\documentclass[10pt,twocolumn,letterpaper]{article}
\usepackage{iccv}

\usepackage{times}
\usepackage{epsfig}

\usepackage{graphicx,wrapfig}
\usepackage[nopar]{lipsum}

\usepackage{float}

\usepackage{amsmath}
\usepackage{amssymb}
\usepackage{booktabs}

\usepackage{multirow}

\usepackage{amsthm}

\usepackage{amsfonts}
\usepackage{amsmath}
\newcommand{\R}{\mathbb{R}}
\renewcommand{\dot}[2]{\big\langle#1,#2\big\rangle}

\newcommand{\grad}[1]{\nabla#1}
\newcommand{\abs}[1]{\left|#1\right|}
\newcommand\norm[1]{\left\lVert#1\right\rVert}
\newcommand{\lossPDE}{\mathcal{L}_{\text{LSE}}}
\newcommand{\lossdata}{\mathcal{L}_{\text{data}}}

\newenvironment{psmallmatrix}
  {\left(\begin{smallmatrix}}
  {\end{smallmatrix}\right)}

\usepackage[pagebackref=true,breaklinks=true,letterpaper=true,colorlinks,bookmarks=false]{hyperref}

\usepackage[capitalize]{cleveref}
\crefname{section}{Sec}{Secs.}
\Crefname{section}{Section}{Sections}
\Crefname{table}{Table}{Tables}
\crefname{table}{Tab.}{Tabs.}

\iccvfinalcopy 

\def\iccvPaperID{9254} 

\ificcvfinal\fi

\begin{document}

\title{
Neural Implicit Surface Evolution
}

\author{
\normalsize Tiago Novello\\
\small IMPA
\and
\normalsize Vinicius da Silva\\
\small PUC-Rio
\and
\normalsize Guilherme Schardong\\
\small U Coimbra
\and
\normalsize Luiz Schirmer\\
\small Unisinos
\and
\normalsize Helio Lopes\\
\small PUC-Rio
\and
\normalsize Luiz Velho\\
\small IMPA
}

\maketitle
\ificcvfinal\thispagestyle{empty}\fi

\begin{abstract}
This work investigates the use of smooth neural networks for modeling dynamic variations of implicit surfaces under the \textbf{level set equation} (LSE). For this, it extends the representation of neural implicit surfaces to the space-time $\R^3\!\!\times\! \R$, which opens up mechanisms for \textbf{continuous} geometric transformations. Examples include evolving an initial surface towards general vector fields, smoothing and sharpening using the mean curvature equation, and interpolations of initial conditions.

The network training considers two constraints. A \textit{data} term is responsible for fitting the initial condition to the corresponding time instant, usually $\R^3\!\! \times\! \{0\}$. Then, a \textit{LSE} term forces the network to approximate the underlying geometric evolution given by the LSE, \textbf{without any supervision}.
The network can also be \textbf{initialized based on previously trained initial conditions}, resulting in faster convergence compared to the standard approach.
\end{abstract}

\section{Introduction}
\label{s-introduction}
A \textit{neural implicit function} $g\!:\!\R^3\!\!\to\! \R$ is a smooth neural network that represents an implicit function. Since $g$ is smooth, objects from the \textit{differential geometry} of its \textit{regular level sets} can be used in closed form~\cite{yang2021geometry, novello21diff}.

This work investigates the extension of the domain of neural implicit functions to the \textit{space-time} $\R^3\!\times\!\R$, encoding the evolution of the function $g$ as a higher-dimensional function $f\!:\!\R^3\!\!\times\!\! \R\! \to\! \R$.
The resulting animation is governed by a PDE, the \textit{level set equation} (LSE) $\frac{\partial f}{\partial t}\!=\!v\norm{\nabla f}$, which encodes the propagation of the level sets $S_t$ of $f(\cdot, t)$ towards their normals with speed $v$.
The choice of the function $v$ depends on the underlying geometric model.
LSE is an important tool for \textit{geometry processing} applications.

We propose to use a \textit{neural network} $f\!:\!\R^3\!\times\! \R\! \to\! \R$ to represent the above level set function. For this, we train $f$ to learn the evolution $S_t$ of an initial surface $S$ under a given LSE.
Accordingly, we add the constraint $f\!=\!g$ on $\R^3\!\times\!\{0\}$, to the LSE. If~$g$ is the \textit{signed distance function} (SDF) of $S$, then a solution $f$ of this problem encodes an animation~of~$S$. 
\pagebreak

\noindent Our~method is the first neural approach in geometry processing, that does not consider numerical approximations of the LSE solution during sampling and does not discretize the LSE in the loss~function.

Our strategy has two steps to train $f$.
First, a sample $\{p_i, g(p_i)\}$ enables us to train the initial condition $f\!=\!g$ on $\R^3\!\times\!\{0\}$. Second, an LSE constraint is used to fit $f$ into a solution of the LSE in $\R^3\!\times\! \R$.
This term does not need any supervision, i.e. it does not consider samples of $f$. This constraint \textit{only} uses samples in the form $(p_i,t_i)$.
The only requirement of our method is that $f$ must be smooth.
Given this, our main contributions can be summarized as follows:
\begin{itemize}
\item Extension of neural implicit surfaces to space-time without the use of numerical/discrete approximations. Encoding the space-time coordinates as input of the network allows us to represent the whole implicit animation in a single network;

\item Development of a neural framework to fit solutions of an LSE using only its analytical expression~(Sec~\ref{s-training}). Moreover,  the network training considers only the initial condition of the LSE problem  as data;
\item The method is flexible enough to be used in a variety of applications, such as surface motion by vector fields, smoothing, sharpening, and interpolation (Sec~\ref{s-experiments});
\item We also propose a novel network initialization based on previously trained initial conditions (Sec~\ref{s-network-initialization}).
\end{itemize}

\section{Related Works}
\label{s-related_works}
Many problems, such as shape correspondence, topology changes, and animation of deformable objects, can be posed using implicit surfaces~\cite{warp_morph_book, turk_99_shape_xform, ls_meta_01, desbrun_deformable, desbrun_soft, desbrun1998active}.
Studying their shape properties through differential geometry leads to a framework for intrinsic operations.
An example is smoothing a surface using the \textit{mean curvature equation}, an important PDE in geometry processing~\cite{mean_curv,clarenz2000anisotropic, desbrun1999implicit, desbrun1998active}.
Problems in this topic rely on computing derivativess -- a hard task when dealing with meshes~\cite{Crane_DDG_2013, DeGoes_DDG_2020, DeGoes_EC_2016}.
A practical neural implicit approach would allow us to compute such objects in closed form, and it is the objective of this work.

Several works showed that modeling surfaces as level sets of neural networks result in a compact representation \cite{Park_2019_CVPR, ISNN_2019, gropp2020implicit, sitzmann2020implicit,novello21diff, silva21mipplicits}.
 Most of them fit a network~into~data.

\noindent In the SDF case, a regularizer term forces the network to satisfy the \textit{Eikonal equation}.
The robustness of those approaches is our motivation to study the evolution of neural implicit surfaces using the level set equation~\cite{osher1988fronts, osher2004level}, a PDE widely used in geometry processing~\cite{desbrun1998active, bertalmio2001variational, kimmel1998computing, sethian1999level, malladi1995shape, whitaker1995algorithms}.

NFGP~\cite{yang2021geometry} and NIE~\cite{mehta2022level} are recent neural approaches to evolve implicit surfaces. 
Both store the evolution in a sequence of networks, each representing a time step, as opposed to our method, which encodes it in a single network.
Specifically, they use a network $g_\phi$ to fit an initial function~$g$. Then, $\phi$ is updated at each time step creating a sequence of networks $g_{\phi_i}\!\!:\!\R^3\!\!\to\! \R$.
This is similar to the Runge-Kutta methods but, instead of fitting the solution to a grid, they use a network.
Thus, for them to evaluate at intermediate times, they have to retrain the networks.
In contrast, our method does not discretize time, learning the solution in a continuous interval using a single network with a domain in $\R^3\times \R$.

NIE is the only approach that evolves $g_\phi$ using the mean curvature equation $\frac{\partial f}{\partial t}\!\!\!=\!\!\norm{\nabla{f}} \kappa$.
However, this evolution uses {a discretization of the partial derivative $\frac{\partial f}{\partial t}$}.
To update~$\phi$, NIE uses a finite difference scheme to approximate the discrete solutions.
To compute the mean curvature $\kappa$, NIE extracts the level sets using marching cubes and employs the cotangent Laplacian, which is problematic since it depends on approximating the level sets by meshes using marching cubes. Also, this operator does not preserve the Laplacian
natural properties -- the \textit{no free lunch scenario}~\cite{wardetzky2007discrete}.
In contrast, our approach uses the network high-order derivatives to evaluate the LSE analytically. For example, we compute the curvature using the \textit{divergence} of $\frac{\nabla f}{\norm{\nabla f}}$.
To evaluate $\text{div}\frac{\nabla f}{|\nabla f|}$ and $\nabla f$  we simply we use \textit{automatic differentiation}.
Moreover, NIE cannot consider multiple initial conditions as our approach does.

Recently, there has been a growing interest by the \textit{physical simulation} community in solving PDEs using neural~networks.
\textit{Physically informed neural networks} (PINNs) are established approaches in this context.
Unlike our proposal, this method~\cite{raissi2017physics} relies on measurements of the PDE solution at intermediate times.
PINNs are extensively evaluated in surveys~\cite{karniadakis2021physics, cuomo2022scientific}.
Karniadakis et al.~\cite{karniadakis2021physics} reviews \textit{inverse problems}, which try to infer the PDE parameters based on supervised data.
This context differs from ours since we do not address the inverse problem, they also do not consider implicit surface evolution using the LSE, and we do not rely on supervised data of the evolution.
%
%

Cuomo et al.~\cite{cuomo2022scientific} surveys a broader range of problems based on the PDE type. We focus on the evolution of implicit surfaces using the LSE, which is a geometric \textit{time-dependent PDE} not explored in their review~\cite[Sec 3.2.2]{cuomo2022scientific}.
Thus, to the best of our knowledge, there is no PINN-based approach to solve the LSE for implicit surface evolution.

Our proposal seeks to bridge this gap by leveraging the representation capacity of (sinusoidal coord-based) networks to solve the above geometric problem without any measurements of the PDE solution at intermediate times; i.e., we need \textbf{supervision only on the initial surfaces}. This can potentially enable new applications in computer graphics, computer-aided design, and computational geometry.

Regarding the interpolation problem, Liu et al.~\cite{liu2022learning} propose the \textit{Lipschitz MLP}, which regularizes a neural network by penalizing the upper bound of its Lipschitz constant.
We use a specific PDE to interpolate SDFs, resulting in smoother and more natural transitions between shapes (see~Sec~\ref{s-interpolation_between_surfaces}).
Moreover, our method manages to use smaller architectures by considering sinusoidal MLPs.

\section{Background and conceptualization}
\label{s-conceptualization}

\subsection{Implicit surfaces}
The level set $g^{-1}\!(0)$ of a smooth function $g\!:\!\R^3\!\!\to\!\!\R$ is a (regular) surface if $\nabla g\neq 0$ in $g^{-1}(0)$.
Conversely, given a surface $S$, there is a function $g$ having it as its zero-level~set.
Thus, $g$ may be reconstructed from a sample of $S$.
For this, we parametrize $g$ using a neural network.
SIREN~\cite{sitzmann2020implicit} and IGR \cite{gropp2020implicit} are examples of such networks.

To compute the parameters of $g$ such that $g^{-1}(0)\approx S$, it is common to consider the \textit{Eikonal} problem:
\begin{align}\label{e-3d_eikonal_equation}
\norm{\nabla g}=1 \text{subject to }
g =0 \text{on } S.
\end{align}
Which asks for $g$ to be the SDF of a set containing $S$. We can derive from Eq~\eqref{e-3d_eikonal_equation} that $\dot{\nabla g}{N}=1$ on $S$, which implies that $\nabla g$ must be aligned with the normals of $S$.

\vspace{0.2cm}

We are interested in using neural networks to evolve~$S$.
The level sets of $g$ could be used, but they do not allow intersections between surfaces at different instants. To avoid this,
we extend the domain of the implicit function to $\R^3\times \R$, where the parameter $t\in\R$ controls the animation.

\subsection{Evolving implicit surfaces}
We use a function $f:\R^3\times\R\to \R$ to define the above extension. For evolving the level sets of $g$ we require $f=g$ in $\R^3\times \{0\}$ and the resulting \textit{evolution} is given~by
\begin{align}
    S_t = f_t^{-1}(0)=\{p\in\R^3|\, f(p,t)=0\}
\end{align}

We assume that $f$ is negative in the interior of $S_t$ and its \textit{normal} vectors are given by $N=\frac{\nabla f}{\norm{\nabla f}}$. Furthermore, $\nabla f$ denotes the \textit{gradient} of $f$ with respect to the space $\R^3$.

The time $t$ allows continuous navigation in $S_t$ and represents a \textit{transformation} $S_0 \to S_{t} $ of the initial surface $S_0$.

\pagebreak

\noindent Moreover, for $t_0, t_1\in \R$,
the function $f$ provides a \textit{smooth deformation} between $S_{t_0}$ and $S_{t_1}$.
These surfaces can contain singularities as their topologies may change over time. Hart \cite{hart1998morse} has studied this phenomenon using Morse~theory.

Such an Eulerian approach is usually explored in geometry processing by storing $f$ on a 4D grid~\cite{desbrun1998active}.
Here, we take an analytical formulation by parametrizing $f$ by a (coord-based) network.
Such an approach has several advantages.
First, \textit{automatic differentiation} provides us the analytical \textit{derivatives} of $f$, which may be used at the loss function.
We can also compute the normals and curvature measures of $S_t$ in closed form~\cite{novello21diff}.
Additionally, neural networks are compact representations for implicit functions, guaranteed by the \textit{universal approximation~theorem}~\cite{cybenko1989approximation}.
Storing $f$ with precision in a 4D grid could be an unfeasible task.

\subsection{Level set equation}\label{s-level-set}
Encoding the surface evolution $S_t$ by the time-dependent function $f$ results in a PDE -- the \textit{level set equation} (LSE). We describe it below.


Let $g:\!\R^3\!\!\to\! \R$ be the SDF of the initial surface $S$.
A function $f\!:\!\R^3\!\!\times\!\R\!\to\! \R$ encodes the evolution $S_t$ if $f\!=\!g$ on $\R^3\!\times\! \{0\}$, and for each point $p$, the function $f$ is constant along its path $\alpha(t)$, i.e. $f\big(\alpha(t), t\big)\!=\!c$ iff $g(p)\!=\!c$.
Thus, deriving the function $f\big(\alpha(t), t\big)$ we obtain
\begin{align}\label{e-level_set}
\displaystyle\frac{\partial f}{\partial t} \big(\alpha(t), t\big) +\dot{\nabla f\big(\alpha(t), t\big)}{\alpha'(t)}=0.
\end{align}
The derivative $\alpha'(t)$ is a vector field along the path $\alpha(t)$.

As Eq~\eqref{e-level_set} holds for each point $p$, we can drop its path $\alpha$ and use only $\alpha'$, which can be seen as a time-dependent vector field $V(p,t)=\alpha'(t)$. Then, the function $f$ encoding the animation $S_t$ is a solution of the LSE:
\begin{align}\label{e-level_set_equation}
\begin{cases}
\displaystyle\frac{\partial f}{\partial t} +\dot{\nabla f}{V}=0 & \text{in } \R^3\times (a,b),\\
f =g & \text{on } \R^3\times \{t=0\}.
\end{cases}
\end{align}
The time interval $(a,b)$ contains $t=0$ and controls the evolution $S_t$.
A solution $f$ of Eq~\eqref{e-level_set_equation} implicitly encodes the integration of $V$.
Thus, defining a family of vector fields is a way to animate a given surface using the LSE.


Observe that in the same sense that we considered a neural implicit function as a solution of Eq~\eqref{e-3d_eikonal_equation}, we assume that the function $f$ is a solution of $\frac{\partial f}{\partial t} +\dot{\nabla f}{V}=0$ subject to $f=g$ on $\R^3\times \{0\}$.
Sec~\ref{s-examples} presents examples of LSEs.


\section{Method}
\label{s-training}

We propose representing a surface evolution by a (coord-based) neural network $f_\theta:\R^3\times \R\to \R$.
Given an LSE with initial conditions, this section defines a machine learning framework, consisting of a loss functional, sampling strategies, and a network initialization, to train $f_\theta$ to approximate a solution to the LSE problem.

Let $g_i:\R^3\to \R$ be the SDFs of $n$ surfaces $S_i$.
We train $f_\theta$ by forcing it to approximate a solution a \textit{neural} LSE:
\begin{align}\label{e-PDE}
\begin{cases}
\mathcal{F}:=\displaystyle\frac{\partial f_\theta}{\partial t} +\dot{\nabla f_\theta}{V}=0 & \text{in } \R^3\times (a,b),\\
f_\theta =g_i & \text{on } \R^3\times \{t_i\}.
\end{cases}
\end{align}
We employed the notation $\mathcal{F}$ to represent the LSE for brevity.
The untrained network $f_\theta$ must encode the movement ruled by the vector field $V$.
$(a,b)$ can be used to control the resulting neural animation $S_t$ of $S$.

\subsection{Loss functional}
\label{ss-loss_functional}

We use Eq~\eqref{e-PDE} to define a loss function to train $f_\theta$.
\begin{align}\label{e-PDE_loss_function}
\mathcal{L}(\theta)=\underbrace{\int\limits_{\R^3\times(a,b)}\!\!\!\big|\mathcal{F}\big|dpdt}_{\mathcal{L}_{\text{LSE}}(\theta)} +\underbrace{\sum_{i=1}^n\int\limits_{\R^3\times\{t_i\}} \!\!\!\big|f_\theta-g_i\big|dp}_{\mathcal{L}_{\text{data}}(\theta)}.
\end{align}
The LSE \textit{constraint} ${\mathcal{L}_{\text{LSE}}}$ forces the network $f_\theta$ to satisfy $\mathcal{F}=0$ and works as a regularization of $f_\theta$ that requires it to follow the underlying phenomenon.
The \textit{data constraint} $\mathcal{L}_{\text{data}}$ asks for $f_\theta$ to satisfies $f_\theta=g_i$ on $\R^3\times \{t_i\}$.

\subsection{Sampling}
\label{ss-training}
To approximate a solution $f_\theta$ of Eq~\eqref{e-PDE}, we seek a minimum $\theta$ of the loss function $\mathcal{L}$ using the \textit{stochastic gradient descent}. For this, we enforce $\mathcal{L}={\mathcal{L}_{\text{LSE}}}+\mathcal{L}_{\text{data}}$ with a sampling in $\R^3\times(a,b)$ and another in $\bigcup_{i=1}^{n}\R^3\times\{t_i\}$.

\paragraph*{Sampling space-time points}~\\
During training, we sample minibatches of $l_1\in\mathbb{N}$ \textit{space-time points} $(p_j,t_j)\in{\R^3\times (a,b)}$ randomly. Then, $\mathcal{L}_{\text{LSE}}$ is enforced in $(p_j,t_j)$, yielding the approximation:
\begin{align*}
\widetilde{\mathcal{L}_{\text{LSE}}}(\theta)=\frac{1}{l_1}\sum_{j=1}^{l_1}\abs{\mathcal{F}(p_j,t_j)}
\end{align*}

\noindent Observe that the LSE constraint $\mathcal{L}_{\text{LSE}}$ does not need any data supervision.

\paragraph*{Sampling initial conditions}~\\
The data constraint $\mathcal{L}_{\text{data}}$ forces $f_\theta$ to fit the input dataset, which we consider to be the SDFs $g_i$ of $n$ surfaces $S_i$.

We use Eq~\eqref{e-3d_eikonal_equation} to define $\mathcal{L_{\text{data}}}=\sum\mathcal{L}_{i}$, with $\mathcal{L}_{i}$ managing the restrictions $g_i=f_\theta$ on $\R^3\times \{t_i\}$.
\begin{align*}
    \mathcal{L}_{i} \!\!=\!\!\!\!\!\! \underbrace{\int\limits_{\R^3\times{t_i}}\!\!\!\!\Big|1\!-\!\big|\nabla f_\theta\big|\Big|dp}_{\mathcal{L}_{\text{Eikonal}}}+\!\!\!\!
    \underbrace{\int\limits_{\R^3\times{t_i}}\!\!\!\!\!\!\big|f_\theta \!-\!g_i\big|dp}_{\mathcal{L}_{\text{Dirichlet}}}+\!\!\!
    \underbrace{\int\limits_{S_{i}}\!\big|1\!-\!\dot{\nabla f_\theta}{N_i}\big|dS_i}_{\mathcal{L}_{\text{Neumann}}}.
\end{align*}
Where ${\mathcal{L}_{\text{Dirichlet}}}$ asks for $f_\theta$ to fit $g_i$ at time $t_i$, ${\mathcal{L}_{\text{Neumann}}}$ requires the alignment between $\nabla f_\theta$ and the normals of $S_i$, and ${\mathcal{L}_{\text{Eikonal}}}$ is the Eikonal regularization.
During training, these constraints are discretized, as in the PDE constraint case.
Then, we sample \textit{minibatches} with $l_2$ \textit{on-surface} points ($g_i=0$) and $l_3$ \textit{off-surface} points ($g_i\neq 0$).
Observe that only the initial conditions $\{g_i\}$  are used in $\mathcal{L}_{\text{data}}$.

In practice, we use two kinds of initial conditions. First, the neural networks $g_i$ fit the SDFs of $S_i$, resulting in faster training since, for each point $p_i$, the values $g_i(p_i)$ and $\grad{g_i}$ given by the evaluation of $g_i$ and its derivative~at~$p_i$.
In this case, we avoid using ${\mathcal{L}_{\text{Eikonal}}}$ since $g_i$ are already trained to satisfy the Eikonal equation.
Second, we consider point clouds $\{p_j, N_j\}_i$ sampled from $S_i$, where we have to approximate the SDF of $S_i$~\cite{novello21diff}.
In both cases, we include the constraint ${\mathcal{L}_{\text{Neumann}}}$ that forces a \textit{normal alignment} at $t=0$ regularizing the orientation near the zero-level set.

During training, we sample minibatches of size $l_1\!+\!l_2\!+\!l_3$ to feed $\mathcal{L}$. $l_i$ are the numbers of space-time, on-surface, and off-surface points. The experiments shown good results using that $l_1$, $l_2$, $l_3$ have $50\%$, $25\%$, $25\%$ of the minibatch size.
The supplementary materials give experiments~varying $l_i$.

\subsection{Neural network architecture}
\label{ss-nn_architecture}

We consider the neural network to be a \textit{sinusoidal MLP}
$f_\theta(p)=W_{d+1}\circ f_{d}\circ \cdots \circ f_{1}(p)+b_{d+1},$
with $d$ hidden layers $f_i(p_i)=\sin(W_ip_i\!+\!b_i)$, where $W_i\in\R^{N_{i+1}\times {N_i}}$ are the weight matrices, and $b_i\!\in\!\R^{N_{i+1}}$ are the biases. The sine is applied at each coordinate of $W_ip_i+b_i$. $\theta$ consists of the union of the coefficients of $W_i$ and $b_i$.
The integer $d$ is the \textit{depth} of $f_\theta$ and the dimensions $N_i$ are the layers \textit{widths}.

The network $f_{\theta}$ is smooth and we can compute its derivatives using automatic differentiation.
Therefore, we train $f_\theta$ using the loss function $\mathcal{L}$ in Eq~\eqref{e-PDE_loss_function}.

\subsection{Network initialization}
\label{s-network-initialization}

We introduce a novel initialization of $f_\theta:\R^3\times\R\to \R$ based on a previously trained network $g_\phi:\R^3\to \R$. This initialization of $\theta$ using $\phi$ results in faster training compared with the standard definitions \cite{sitzmann2020implicit} (see Sec~\ref{s-initialization_experiments}).

Assume that the training of $f_\theta$ is subject to
$f_\theta =g_\phi$ on $\R^3\times \{0\}$. Then
we define $\theta$ in terms of $\phi$ such that $f_\theta(p,t)=g_\phi(p)$ for all $t$, that is, $f_\theta$ will be constant and equal to $g_\phi$ over time.
This allow us to start the training of $f_\theta$ to fit a solution of an underlying LSE problem with $f_\theta$ already satisfying the initial condition.

For this, we suppose that $f_\theta$ is wider than $g_\phi$ and that their depths are equal to $d$.
Specifically, let $B_i\in\R^{N_{i+1}\times N_i}$, $b_i\in\R^{N_{i+1}}$ be the \textit{trained} weight matrices and biases of $g_\phi$, and $A_i\in\R^{M_{i+1}\times M_i}$, $a_i\in\R^{M_{i+1}}$ be the \textit{untrained} weight matrices and biases of $f_\theta$.
Since $f_\theta$ is wider than $g_\phi$, i.e. $N_i\leq M_i$, we can define $A_i$, $b_i$ using

\vspace{-12pt}

\begin{align*}
    A_1=\begin{psmallmatrix}B_1 & 0\\F_p & F_t\end{psmallmatrix}, & \quad A_i=\begin{psmallmatrix}B_i & 0\\0 & 0\end{psmallmatrix} \text{for } i=2,\ldots, d,\\ A_{d+1}=\begin{psmallmatrix}B_{d+1} & L\end{psmallmatrix},&\quad  a_i=\begin{psmallmatrix}b_i \\0\end{psmallmatrix} \text{for } i=1,\ldots, d+1.
\end{align*}

\vspace{-2pt}

\noindent Thus, $f_\theta(p,t)=g(p)$ for all $t$, see supp. material for the details.
$F_p$, $F_t$ project the input $(p, t)$ in the dictionary $\sin\left(F_p p + F_t  t\right)$, and are initialized using the standard approach.
Note that these sines are not used in the first training step, but as it advances, the new hidden weights combine them improving the training (see Sec~\ref{s-initialization_experiments}).


\section{Examples}
\label{s-examples}
Here, we present examples of neural implicit evolution using LSE.
Sec~\ref{ss-time-independent_vector_fields} shows examples using time-independent vector fields.
Sec~\ref{ss-mean_curvature_flow} considers the mean curvature equation, which is intrinsically related to the surface and results in smoothing/sharpening applications for implicit neural surfaces. Sec~\ref{ss-interpolation} investigates interpolations between implicit neural surfaces using an LSE.

Recently, these problems have been addressed by different neural methods~\cite{mehta2022level, yang2021geometry, liu2022learning}. Comparisons are made in Sec~\ref{s-experiments}. Hereafter, we give each application's conceptualization and corresponding loss function.

\subsection{Time-independent vector fields}
\label{ss-time-independent_vector_fields}
Moving a surface $S$ towards a vector field $V:\R^3\to \R^3$
results in a simple LSE.
Specifically, let $g$ be the SDF of $S$.
Since $V$ does not change over time, it may be defined and customized beforehand. For example, sources, sinks, saddles, and constant vectors may be used to generate vector fields based on specific applications.

We train a neural network $f_\theta:\R^3\times\R\to \R$ to implicitly encode the evolution of $S$ by $V$ using the resulting LSE:
\begin{align}\label{e-vector-field}
\begin{cases}
\displaystyle\frac{\partial f_\theta}{\partial t} -{v}\norm{\nabla f_\theta}=0 & \text{in } \R^3\times (a,b),\\
f_\theta =g& \text{on } \R^3\times \{0\}.
\end{cases}
\end{align}
Here, $v$ denotes the size of the normal component of $V$, that is, $v(p,t)=\dot{V(p)}{N_t(p)}$.
The minus in Eq~\eqref{e-vector-field} is because we need the inverse of the resulting flow to compose with~$g$.

Sec~\ref{s-deformation_vector fields} gives two experiments using time-independent vector fields as a proof of concept.
Other examples are presented in the video supplementary material.

\subsection{The mean curvature equation}
\label{ss-mean_curvature_flow}

The \textit{mean curvature equation} evolves the level sets with velocity given by the negative of their mean curvature, resulting in a smoothing along the time~\cite{mean_curv,bellettini2014lecture}.

Let $V(p,t)=-\kappa(p,t)N(p,t)$ be the \textit{mean curvature vector}, where $N$ is the normal field of the level sets and $\kappa=\,\text{{div}}\,N$ is the \textit{mean curvature}; $\text{{div}}$ is the \textit{divergence} operator.
Replacing $V$ in Eq~\eqref{e-PDE} results in:
\begin{align}\label{e-mean_curvature_equation}
\begin{cases}
\displaystyle\frac{\partial f}{\partial t} -\alpha\norm{\nabla f}\,\kappa_\theta =0 & \text{in } \R^3\times (a,b),\\
f =g & \text{on } \R^3\times \{t=0\}.
\end{cases}
\end{align}
Intuitively, the zero-level set moves toward the mean curvature vector $-\kappa N$, contracting regions with positive curvature and expanding regions with negative curvature. Thus, such procedure \textit{smooths} (\textit{sharpens}) the surface if $t>0$ ($t<0$). 
$\alpha$ controls the level set evolution, which has relations with \textit{minimal surfaces} (Sec 1 of supp. material).

Analogously to Eq.~\eqref{e-PDE}, we define $\lossPDE+\lossdata$ to fit a solution of Eq.~\eqref{e-mean_curvature_equation} using $\mathcal{F}:=\frac{\partial f}{\partial t} -\alpha\norm{\nabla f}\,\kappa_\theta$.
Sec~\ref{s-smoothing_sharpening_surface} presents smoothing and sharpening using this technique.

\subsection{Interpolation between implicit surfaces}
\label{ss-interpolation}
Let $g_i$ be the SDFs of two surfaces $S_i$.
We present a LSE approach to interpolate $g_i$.
A vector field $V$ for Eq~\eqref{e-PDE}, such that its solution interpolates between $S_i$, has the form:
\begin{align}\label{e-vector_interpolation}
V(p,t)=-\big(g_2(p)-f(p,t)\big)\frac{\nabla f(t,p)}{\norm{\nabla f(t,p)}},
\end{align}
with $f(p,0)=g_1(p)$. Note that the evolution towards $V$ forces each $c$-level set of $g_1$ to match the $c$-level set of $g_2$.
The resulting LSE is given by substituting Eq~\eqref{e-vector_interpolation} in Eq~\eqref{e-PDE}
\begin{align}\label{e-interpolation_vector_fields}
\begin{cases}
\displaystyle\frac{\partial f}{\partial t} -\norm{\nabla f}(g_2-f) =0 & \text{in } \R^3\times \R,\\
f =g_i & \text{on } \R^3\times \{t_i\}.
\end{cases}
\end{align}
A solution $f$ of Eq~\eqref{e-interpolation_vector_fields} will locally inflate $S_1$ if inside $S_2$, and deflate it if outside so that $S_1$ will always try to fit into $S_2$~\cite{desbrun1998active}.
Again, we define a loss function $\lossPDE + \lossdata$ to fit a solution of Eq~\eqref{e-interpolation_vector_fields} using $\mathcal{F}:=\frac{\partial f}{\partial t} -\norm{\nabla f}(g_2-f)$.

\vspace{0.1cm}

Theoretically, we could use the mean curvature equation to minimize deformations along the interpolation. This LSE has the property of minimizing area distortions of the resulting evolution; see Sec 1 of the supplementary material.

\section{Experiments}
\label{s-experiments}
Here, we present the experiments of the examples~given in Sec~\ref{s-examples}.
See the supplementary material for an ablation study of the training, sampling, and initialization.

\subsection{Deformation driven by vector fields}
\label{s-deformation_vector fields}

We use the the definitions in Sec~\ref{ss-time-independent_vector_fields} to train an animation based on vector fields spatially related to the initial surface.

First, consider the twist $V(x,y,z)=y(-z, 0, x)$ of $\R^3$ along the $y$-axis.
Substituting it in Eq~\eqref{e-PDE} results in a level set equation, which we use to derive a loss function.

Let $g$ be the SDF of the Armadillo and $f_\theta$ be
a network with $2$ hidden layers $f_i:\R^{256}\to \R^{256}$.
We trained $f_\theta$ during $48000$ epochs using minibatches of $25000$ on-surface points ($g=0$), $25000$ off-surface points ($g\neq0$), and $8000$ points in $\R^3\times[-1,1]$.
Fig~\ref{f-twist_armadillo} shows $3$ reconstructions of the zero-level sets of $f_\theta$ at times $t_i=0, 0.25, 0.5$.
\begin{figure}[ht]
    \centering
        \includegraphics[width=0.31\columnwidth]{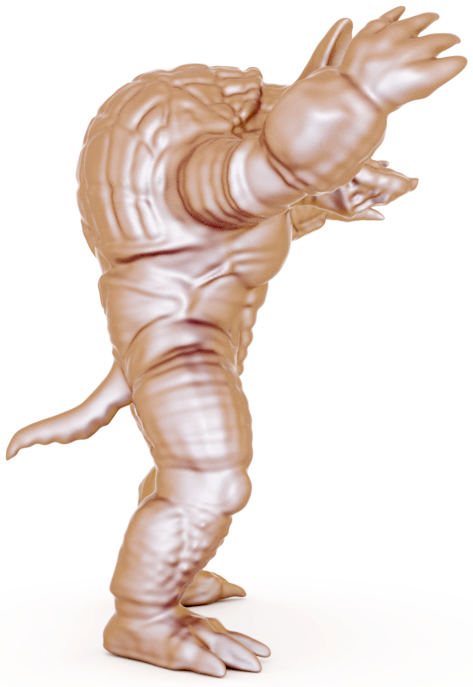}
        \includegraphics[width=0.31\columnwidth]{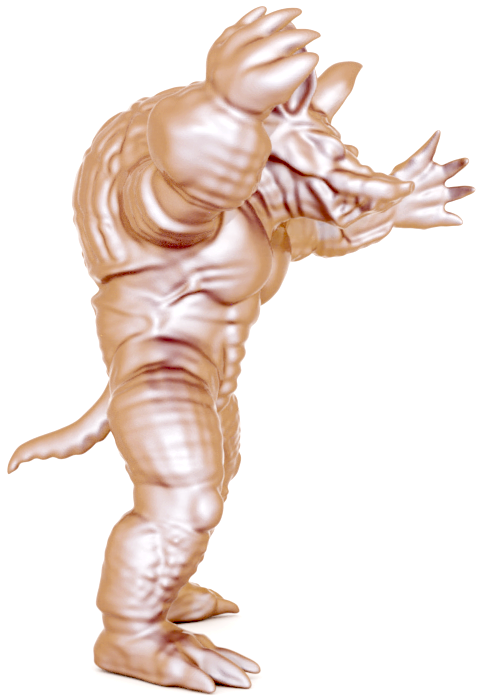}
        \includegraphics[width=0.32\columnwidth]{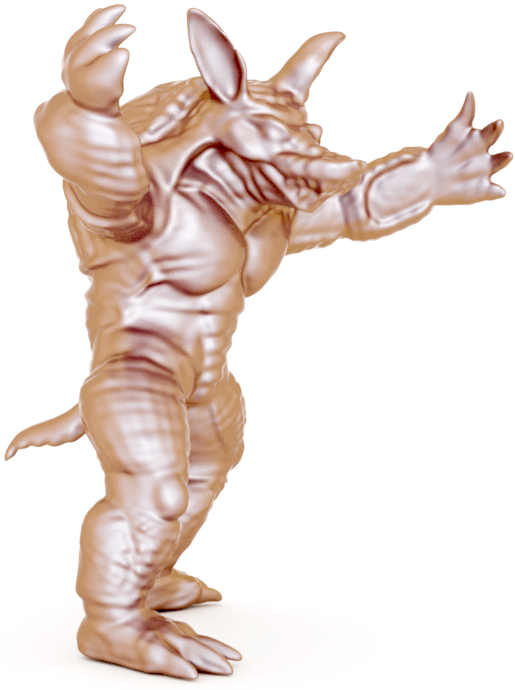}
    \caption{Evolving the level sets of the Armadillo's SDF using the vector field that represents a twist of $\R^3$ along the $y$-axis.}
    \label{f-twist_armadillo}
\end{figure}

Although we do not provide data in $\R^3\!\!\times\!\{t\neq 0\}$, the solution is well approximated (see Fig~\ref{f-twist_armadillo}).
The vertical axis is the $y$-axis, and the origin of $\R^3$ is at the ground.

For the next experiment, let $g$ be the SDF of the Spot (Fig~\ref{f-spot_vector_field}, center).
Define $V=V_1-V_2$ as the sum of a \textit{source} $V_1$ and a \textit{sink} $-V_2$,
with $V_i(p)=e^{-\frac{|{p-p_i}|^2}{0.18}} (p-p_i)$. The~points $p_1$, and $p_2$ are the centers of Spot's body and head.
Again, we use $V$ to derive a loss function to train $f_\theta$.
We parameterize $f_\theta$ with one hidden layer $f_i\!:\!\R^{128}\!\to\! \R^{128}$ and
train it for $46000$ epochs.
As expected, it reduces the Spot's head while it increases the body size, see Fig~\ref{f-spot_vector_field}.
\begin{figure}[ht]
    \centering
        \includegraphics[width=\columnwidth]{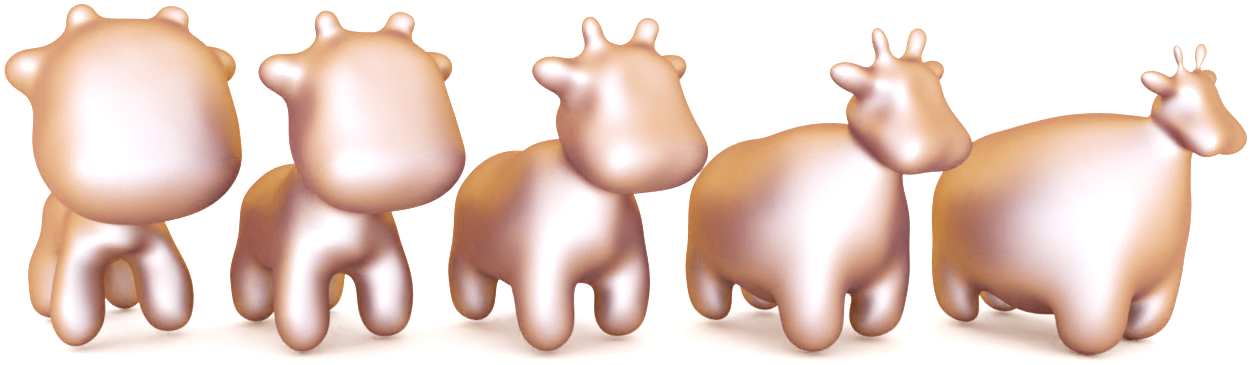}
    \caption{Evolving the zero-level sets of a network according to a vector field with a source and a sink. We set the SDF of the Spot as the initial condition at $t=0$ (middle). The sink/source are inside the head/body of the Spot.
    }
    \label{f-spot_vector_field}
\end{figure}

Table~\ref{t-vector_fields} shows that the above networks are close to satisfying the LSE problems.
We compare $f_\theta(\cdot,0)$ with the initial condition $g$ and measure how close $f_\theta$ is from satisfying $\mathcal{F}=0$. For this we use the following measures: 1) The absolute difference $\abs{f_\theta(\cdot, 0)-g}$ in $\mathbb{R}^3\times\{0\}$; 2) The evaluation of $f_\theta$ in $\abs{\mathcal{F}}$  in $\mathbb{R}^3\times\R$.
We use a sample of $1000$ points in $\mathbb{R}^3\times\{0\}$ and $\mathbb{R}^3\times\R$, not included in the training process, to evaluate the mean/maximum of each measure.
\begin{table}[ht]
\footnotesize
\begin{tabular}{c||ll|ll||ll}
\hline
\multirow{2}{*}{\begin{tabular}[c]{@{}c@{}}Vector\\ field\end{tabular}} & \multicolumn{2}{c|}{\begin{tabular}[c]{@{}c@{}}on-surface\\ constraint\end{tabular}} & \multicolumn{2}{c||}{\begin{tabular}[c]{@{}c@{}}off-surface\\ constraint\end{tabular}} & \multicolumn{2}{c}{\begin{tabular}[c]{@{}c@{}}PDE\\ constraint\end{tabular}} \\ \cline{2-7}
                                                                        & \multicolumn{1}{l|}{mean}                          & max                             & \multicolumn{1}{l|}{mean}                           & max                             & \multicolumn{1}{l|}{mean}                      & max                         \\ \hline
twist                                                                   & 0.0008                                             & 0.003                          & 0.002                                               & 0.028                           & 1e-5                                          & 0.0004                      \\
source-sink                                                                 & 0.0009                                             & 0.005                          & 0.001                                               & 0.015                           & 2.7e-6                                        & 0.0005                      \\ \hline
\end{tabular}
\caption{Comparisons between the ground truth initial conditions at $t=0$ and the measures of how close the trained networks are to satisfy the underlying LSEs.}
\label{t-vector_fields}
\end{table}

\subsection{The mean curvature equation}
\label{s-smoothing_sharpening_surface}
The next set of experiments seek to solve the mean curvature equation. First, we consider simple initial conditions such as the cube and the dumbbell. Then we use the intrinsic properties of this LSE to provide smoothing/sharpening of detailed surfaces.

\subsubsection{Simple initial conditions for validation}
\label{sss-mean_curvature_validation}

Let $g$ be the SDF of the cube, and $f_\theta:\R^3\times \R\to \R$ be a network with $3$ hidden layers $f_i:\R^{256}\to \R^{256}$. We set $\alpha=0.1$ and optimized $f_\theta$ for $8000$ epochs using the loss function resulting from Eq~\eqref{e-mean_curvature_equation}.
We used an oriented point cloud of size $40000$, sampled from the cube.
During training, we consider minibatches of $5000$ on-surface points, $5000$ off-surface points, and $10000$ in $\R^3\times[-1,1]$.

Fig~\ref{f-cube_mean_curv_epoch_8000} shows the level sets of $f_\theta$ at $t=\frac{i}{5}$, $i=0,\ldots,5$.
As expected, regions with positive mean curvature, such as the cube corners, contract.
This LSE evolves the surface toward the normals times the negative of the mean curvature.
Therefore, the cube will at some instant of time collapse to a point, but right before it will be very close to a sphere~\cite{mean_curv}.
\begin{figure}[ht]
    \centering
        \includegraphics[width=\columnwidth]{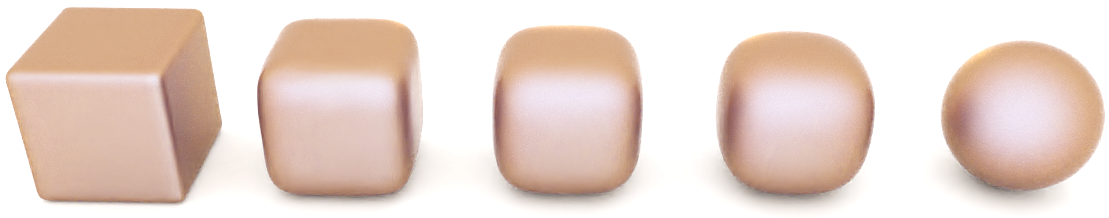}
    \caption{Mean curvature equation of cube surface.}
    \label{f-cube_mean_curv_epoch_8000}
\end{figure}

The \textit{dumbbell} is a classical example.
Let $g$ be its SDF and $f_\theta$ be a network with $2$ layers $f_i:\R^{256}\to\R^{256}$. We set $\alpha=0.05$ and sample a point cloud of size $80000$ from the dumbbell.
The training took $2800$ epochs using minibatches of $5000$ on-surface points, $5000$ off-surface points, and $10000$ points in $\R^3\times[-1,1]$.

Fig~\ref{f-ricci_flow_epoch_2800} shows the level sets of $f_\theta$ at times $t_i=\frac{i}{10}$ for $i=0,\ldots, 7$.
As expected, since the neck region has higher mean curvature, it pinches off first creating two connected components. Later, each component collapses to a point, becoming small spheres right before that. The resulting flow has critical points in different instances of time.
\begin{figure}[ht]
    \centering
        \includegraphics[width=\columnwidth]{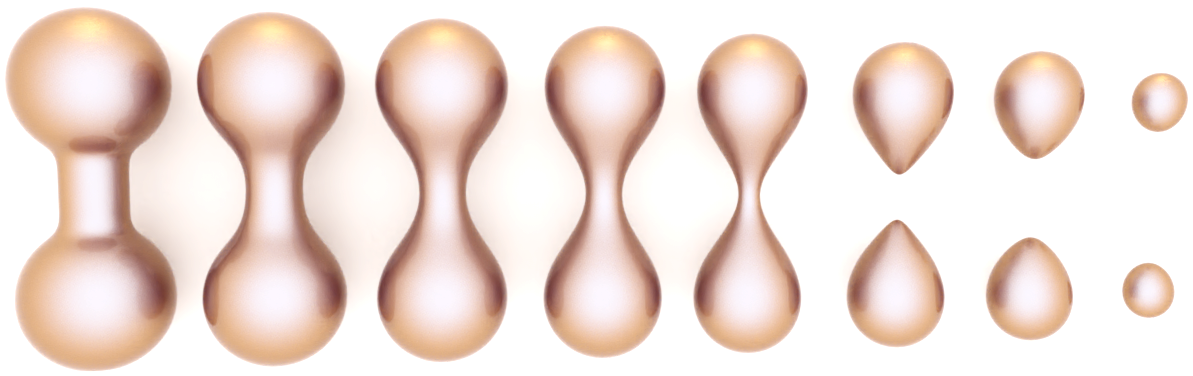}
    \caption{Mean curvature equation of Dumbbell surface.}
    \label{f-ricci_flow_epoch_2800}
\end{figure}

\subsubsection{Smoothing and sharpening}
\label{sss-smoothing_armadillo}

The mean curvature equation evolves the level sets by contracting (expanding) regions with positive (negative) mean curvature. As a consequence, its solution smooths (sharpens) the level sets when $t>0$ ($t<0$).

Let $g$ be the SDF of the Armadillo, and $f_\theta$ be a network with $2$ hidden layers $f_i:\R^{256}\to \R^{256}$. We set $\alpha=0.001$ in Eq~\eqref{e-mean_curvature_equation}. The network $f_\theta$ was trained during $33000$ epochs using an oriented point cloud of size $80000$.
During training, we used minibatches of $5000$ on-surface points, $5000$ off-surface points, and $10000$ in $\R^3\times[-1,1]$.

Fig~\ref{f-armadillo_smoothing_men_curv_a_0.001_epoch_33000} presents three reconstructions of the zero-level sets of $f_\theta$ at times $t=0,0.1,0.2$.
As expected, the Armadillo surface was properly reconstructed at $t=0$ and, as time progressed, it became smoother.
Regions with positive mean curvature, such as the fingers, contracted.
\begin{figure}[ht]
    \centering
        \includegraphics[width=0.9\columnwidth]{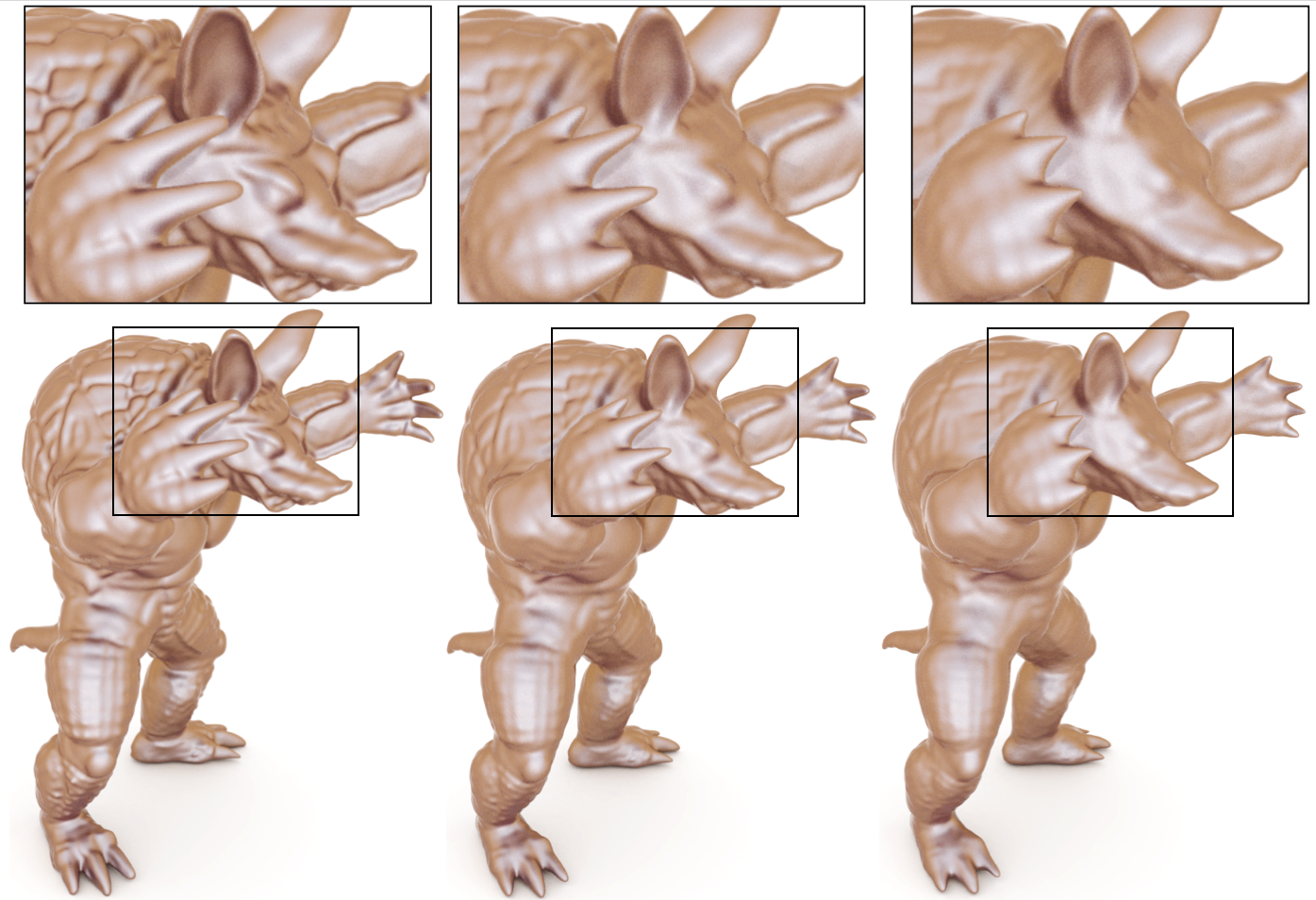}
    \caption{Armadillo smoothing using the mean curvature equation.}
    \label{f-armadillo_smoothing_men_curv_a_0.001_epoch_33000}
\end{figure}

For the sharpening we reconstruct the zero-level sets at $t=0,-0.1,-0.2$ (see
Fig~\ref{f-armadillo_sharpening_men_curv_a_0.001_epoch_33000}).
As expected, regions with positive curvature have expanded, resulting in an enhancement of the geometrical features of the surface.
\begin{figure}[ht]
    \centering
        \includegraphics[width=0.9\columnwidth]{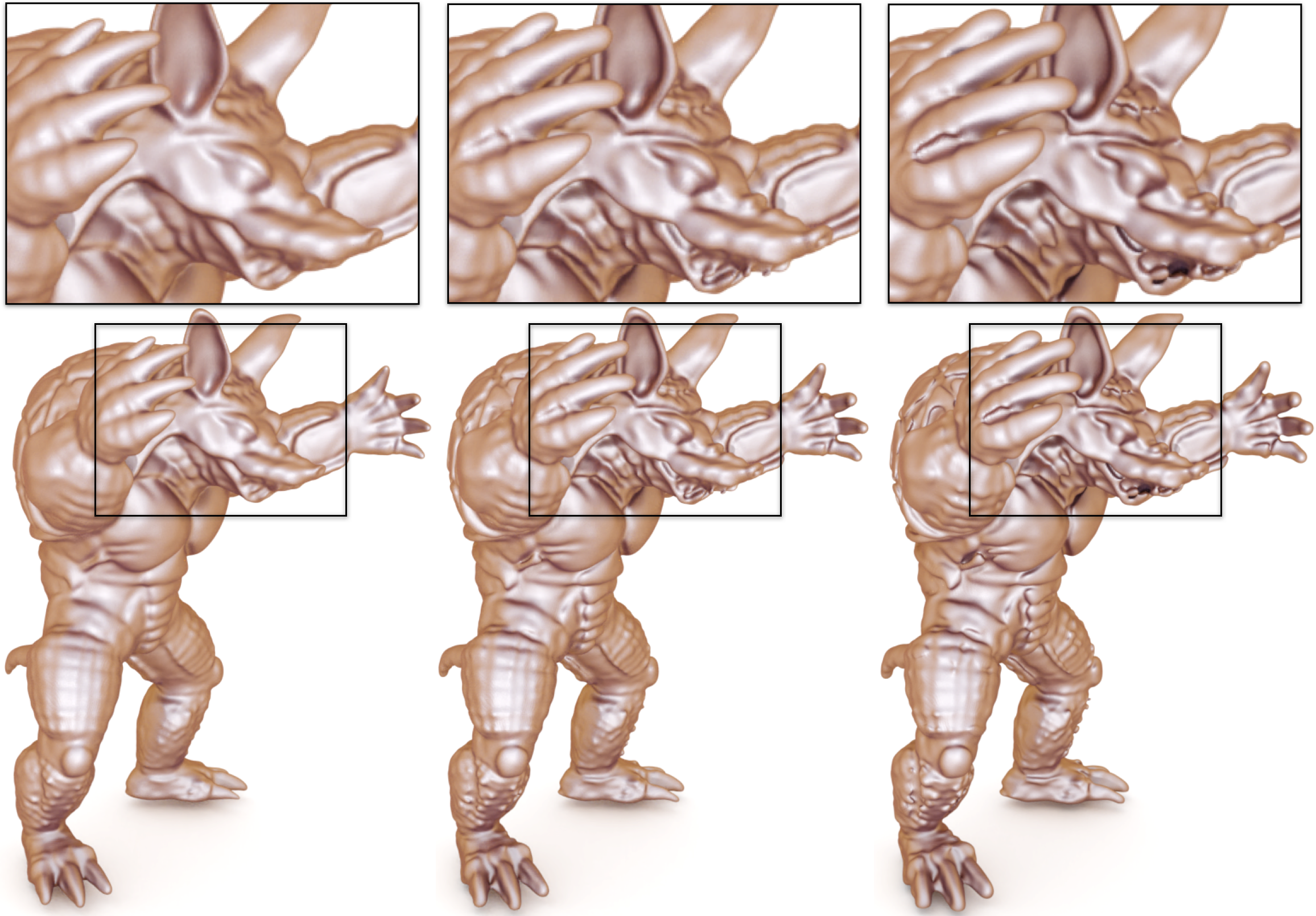}
        \caption{Using the mean curvature equation to enhance the geometrical details of the Armadillo surface.}
    \label{f-armadillo_sharpening_men_curv_a_0.001_epoch_33000}
\end{figure}

\subsubsection*{Numerical evaluation}

There is no available (ground-truth) analytical solutions of the mean curvature equation (Eq~\eqref{e-mean_curvature_equation}) for the above surfaces and finding them is not trivial. Nonetheless, we can quantitatively evaluate the network's proximity to satisfying~Eq~\eqref{e-mean_curvature_equation}.
This is presented in Table~\ref{t-mean_curv_comparisons}.
We compare the trained networks at $t=0$ with the initial surfaces. We also measure how close the networks are to satisfying the LSE.
We used a sample of $1000$ points in $\mathbb{R}^3\times\{0\}$ and $\mathbb{R}^3\times\R$, not included in the training process, to evaluate the mean and maximum values of each measure.
\begin{table}[ht]
\footnotesize
\begin{tabular}{c||ll|ll||ll}
\hline
\multirow{2}{*}{Model} & \multicolumn{2}{c|}{\begin{tabular}[c]{@{}c@{}}on-surface\\ constraint\end{tabular}} & \multicolumn{2}{c||}{\begin{tabular}[c]{@{}c@{}}off-surface\\ constraint\end{tabular}} & \multicolumn{2}{c}{\begin{tabular}[c]{@{}c@{}}PDE\\ constraint\end{tabular}} \\ \cline{2-7}
                       & \multicolumn{1}{l|}{mean}                           & max                            & \multicolumn{1}{l|}{mean}                            & max                            & \multicolumn{1}{l|}{mean}                        & max                       \\ \hline
Cube                   & \multicolumn{1}{l|}{0.0006}                         & 0.007                         & \multicolumn{1}{l|}{0.0013}                          & 0.024                          & \multicolumn{1}{l|}{0.0015}                      & 0.009                     \\
Dumbbell               & \multicolumn{1}{l|}{0.0003}                         & 0.002                         & \multicolumn{1}{l|}{0.0010}                          & 0.013                          & \multicolumn{1}{l|}{0.0009}                      & 0.017                     \\
Armadillo              & \multicolumn{1}{l|}{0.0008}                         & 0.004                         & \multicolumn{1}{l|}{0.0022}                          & 0.016                          & \multicolumn{1}{l|}{0.0019}                      & 0.013                     \\ \hline
\end{tabular}
\caption{Quantitative evaluation of our method in the problem of approximating solutions of the mean curvature equation.
}
\label{t-mean_curv_comparisons}
\end{table}

\subsubsection*{Comparisons}
We compare our technique with NFGP~\cite{yang2021geometry} and NIE~\cite{mehta2022level} for smoothing and sharpening of neural implicit surfaces.

NFGP evolves a network $g_\theta\!:\!\R^3\!\!\to\!\R$ such that the level set of the resulting network $g_\phi$ smooths/sharpens $g_\theta^{-1}\!(0)$. The training optimizes,
$(\kappa_{\phi}\!-\!\beta \kappa_{\theta})^2$, the difference~between the mean curvatures of the level sets of $g_\phi$ and $g_\theta$.
Then, using $\beta\!<\!1$ ($\beta\!>\!1$),
it would force a smoothing (sharpening) of the initial surface.
However, NFGP trains a network $g_\phi$ for each $\beta$, thus it cannot represent a continuous evolution over time. In contrast, our framework directly evolves over time using a single network.
Although the NFGP approach does not use the mean curvature equation model, we can still perform a qualitative analysis as a means of comparison, since a numerical analysis is not feasible.
Fig~\ref{f-sharpening_comparison_ours_nfgp} shows this comparison for sharpening.
The artifacts in the Armadillo's ears are probably due to the inconsistencies in the loss function of NFGP which asks for
$(g_\phi\!-\!{g_\theta})^2$ and $(\kappa_{\phi}\!-\!\beta \kappa_{\theta})^2$, thus the level sets would try to evolve but $(g_\phi\!-\!g_\theta)^2$ forces it to be constant.
\begin{figure}[ht]
    \centering
        \includegraphics[width=0.85\columnwidth]{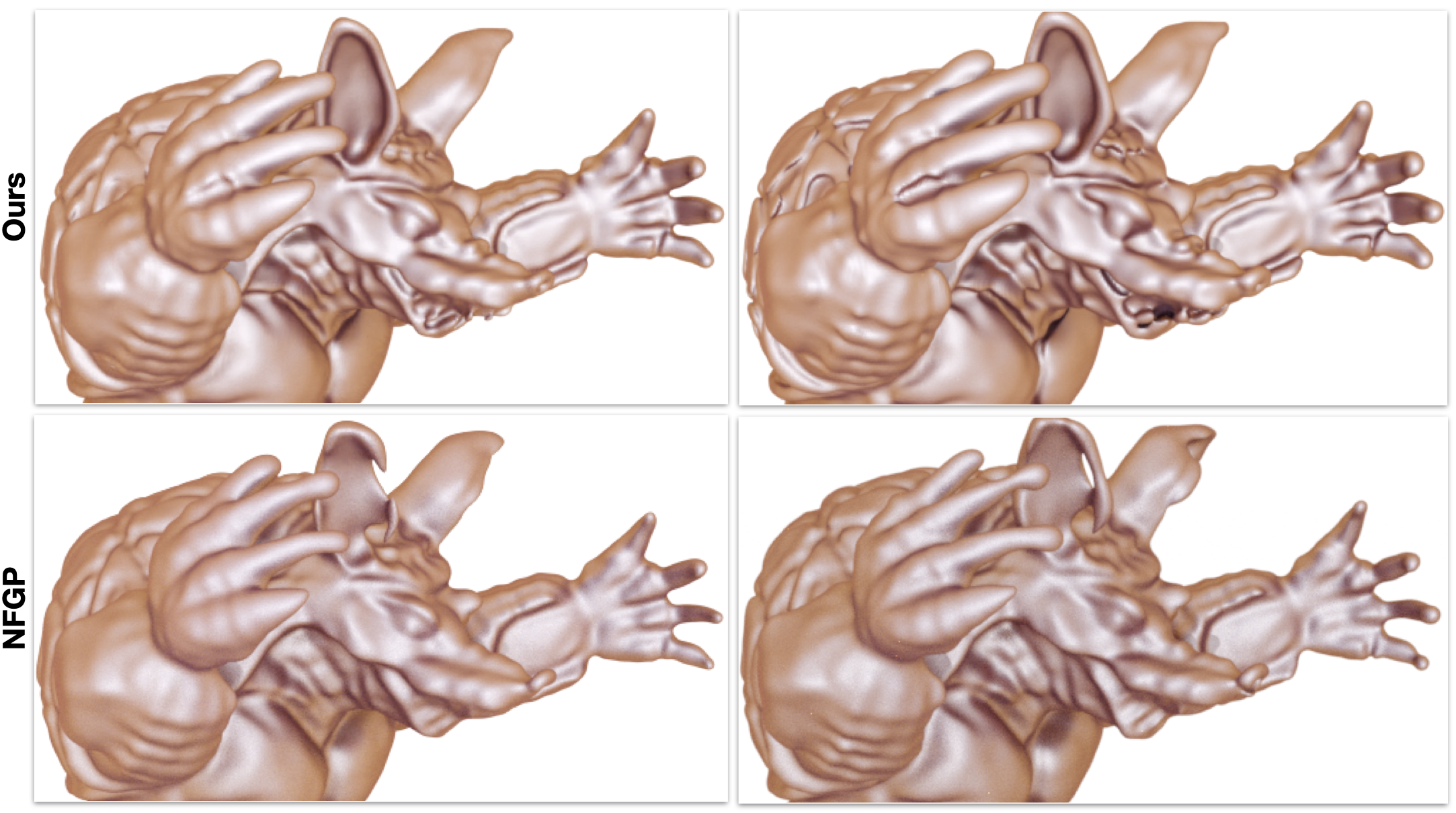}
\caption{Sharpening comparison of our approach with NFGP~\cite{yang2021geometry}. We use the same Armadillos in Fig~\ref{f-armadillo_sharpening_men_curv_a_0.001_epoch_33000}, and $\beta=2, 2.5$ for the NFGP. The experiments used the same initial condition.
Notice that NFGP may produce artifacts while sharpening, as can be seen in Armadillo's ears.
}
    \label{f-sharpening_comparison_ours_nfgp}
\end{figure}

\begin{figure}[ht]
    \centering
        \includegraphics[width=0.85\columnwidth]{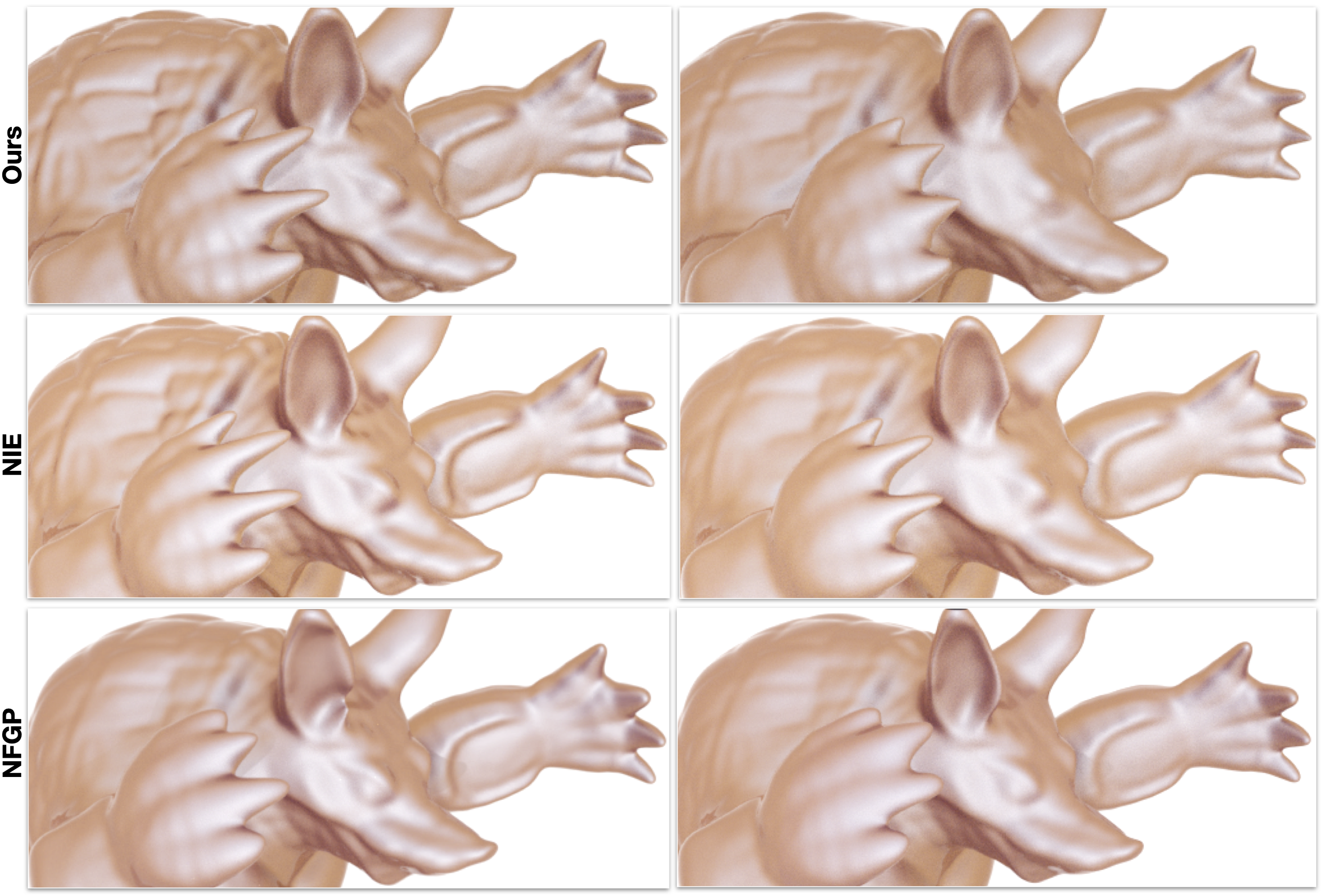}
\caption{Smoothing comparison. Line 1 repeats Fig~\ref{f-armadillo_smoothing_men_curv_a_0.001_epoch_33000}. Line 2 presents steps 4 and 7 of smoothing using NIE~\cite{mehta2022level}. Line 3  shows the results using NFGP~\cite{yang2021geometry} with $\beta=0.95,0.8$. All experiments used the same initial condition.
As expected, our approach is comparable with NIE because they are based on the mean curvature equation. While the overall result of NFGP is smoothing, it may produce artifacts, as can be noticed in the fingers.
}
    \label{f-smoothing_comparison_ours_nie_nfgp}
\end{figure}

We also compare with NIE~\cite{mehta2022level}.
Given a \textit{time step} $\Delta t$, it fits the solution of Eq~\eqref{e-mean_curvature_equation} at times $t_i=i\Delta t$ using the approximation $f_{t_{i+1}} = f_{t_{i}}-\Delta t \alpha \dot{\nabla f_{t_{i}}}{\kappa_\theta N}$.
Thus, for a network $g_{\phi_i}\approx f_{t_{i}}$, NIE trains the next state $g_{\phi_{i+1}}$ by minimizing
$(g_{\phi_{i+1}}-f_{t_{i+1}})^2$. Moreover, a \textit{discrete Laplacian} estimates $\kappa_\theta N$ at the vertices of a mesh approximating~$g_{\phi_i}^{-1}(0)$.
The resulting networks $g_{\phi_i}$ have the domain in $\R^3$ while we use a single network with a domain in $\R^3\times \R$.

Fig~\ref{f-smoothing_comparison_ours_nie_nfgp} shows smoothings of the Armadillo using our method, NIE, and NFGP.
Line 1 repeats the results of Fig~\ref{f-armadillo_smoothing_men_curv_a_0.001_epoch_33000}.
Line 2 gives the results for NIE using $4$ / $7$ steps with~$\Delta t\!=\!1$.
Line 3 presents the results using NFGP to train a network with $\beta\!=\!0.95, 0.8$ defined empirically.
To give a fair comparison, we consider the initial network to be $f_\theta(\cdot, 0)$; $f_\theta$ is the network of our experiment. We use the procedure in Sect. 3 of the Supp. Material to extract $f_\theta(\cdot, 0)$~from~$f_\theta$.

Importantly, we observe no high-frequency noise in the network derivatives.
We used the SIREN extension from \cite{novello21diff} that forces the alignment between the surface normals and the network gradient, differently from NFGP and NIE. 
On the other hand, we observed that computing the mean curvature using the cotangent Laplacian (as in NIE) of a mesh output of a marching cube results in a noisier (see Line 1 of Fig \ref{f-mesh-based-comparison}) compared to our closed form approach (see Line 2 of Fig \ref{f-mesh-based-comparison}) using automatic differentiation.
The first line of Fig \ref{f-mesh-based-comparison} uses the discrete mean curvature (cotangent Laplacian, also used in NIE). Autodiff is applied in the second line without complications.

\subsubsection*{Comparison with mesh-based approaches}

We provide a comparison of our method with a mesh-based approach. We consider the \textit{implicit fairing} method proposed by Desbrun et al.~\cite{desbrun1999implicit}.
Figure \ref{f-mesh-based-comparison} shows the results considering the \textit{implicit fairing} (top) and our approach (bottom) to evolve an Armadillo (with 590k vertices for Desbrun's method) using the \textit{mean curvature flow}, which is the parametric version of the mean curvature equation.
The~colors illustrate the mean curvature.
For the discrete case (Line 1) the curvatures are computed using a mesh-based approach: the \textit{cotangent} formula of the Laplacian of the mesh, also used in NIE.
We run the implicit fairing for 100 steps and extracted three Armadillos.
\begin{figure}[hh]
    \centering
    \includegraphics[width=\columnwidth]{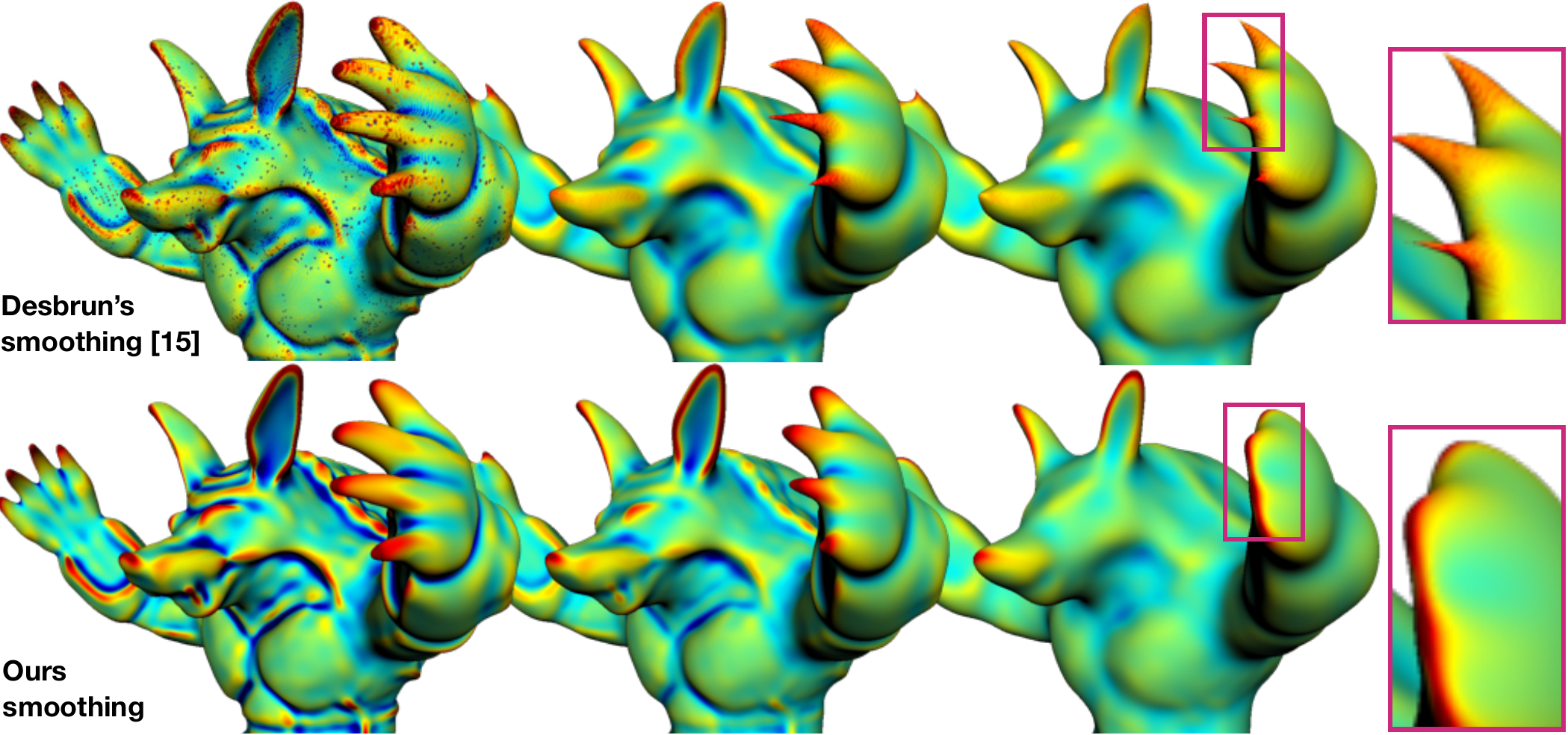}
    \caption{Comparison of our smoothing approach with the implicit fairing~\cite{desbrun1999implicit}. 
  	First line shows the implicit fairing smoothing of the Armadillo triangle mesh. Second line illustrates our approach for smoothing a SDF of the Armadillo.
  	Observe that the mesh-based approach can not properly handle regions with high curvature (such as the Armadillo's fingers). The mean curvature flow should contract those regions like our method does.
}
    \label{f-mesh-based-comparison}
\end{figure}

Evolving the triangle mesh under the mean curvature flow may lead to \textit{shrinkage} in high curvature regions where the surface should contract (see the Armadillo fingers). This results in numerical instabilities since the mesh becomes singular in such regions.

\begin{wrapfigure}[6]{r}{0.51\columnwidth}
\vspace{-0.6cm}
    \includegraphics[width=0.51\columnwidth]{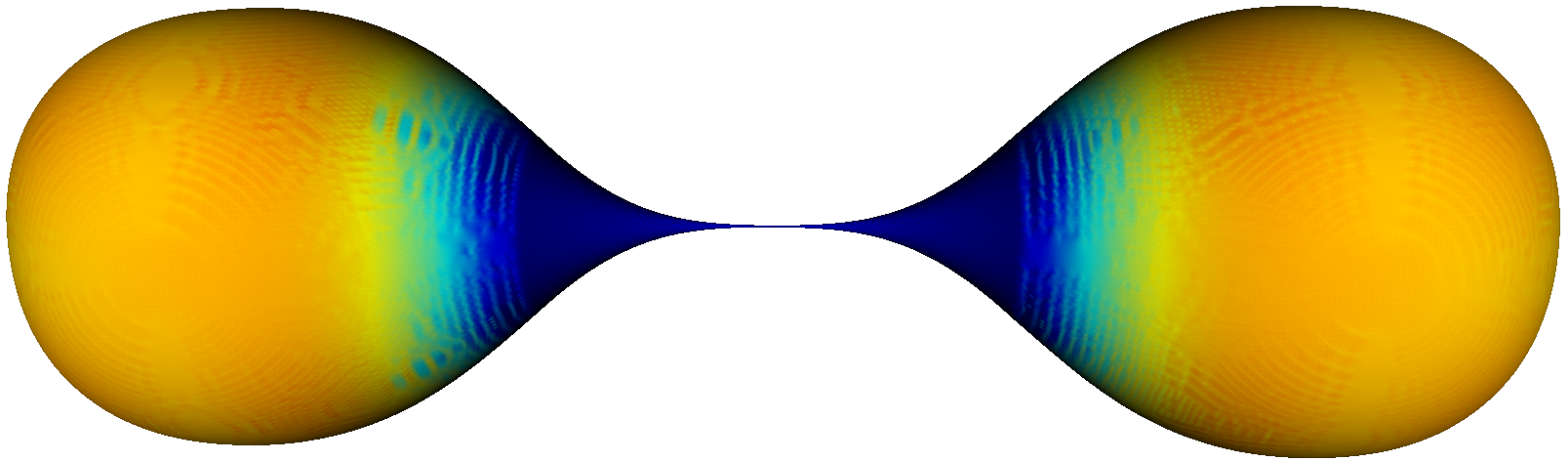}\
    \caption{Evolving a mesh using implicit faring may lead to topological problems (in blue). }
    \label{f-dumbel-change-of-topo}
\end{wrapfigure}

\noindent Another~problem~is the \textit{change of topology} during the surface evolution which may create singularities. Mesh-based methods may encounter problems, such as non-manifolds, in regions affected by topology changes (see the blue region in Fig~\ref{f-dumbel-change-of-topo}).
Our implicit approach overcomes this as shown in Fig~\ref{f-ricci_flow_epoch_2800}.

Training time for our method is comparable to implicit fairing's solution time,
which took 14.4s per iteration and 24.1 min in total to smooth the Armadillo for 100 steps. Our method took 10.1 min for training, allowing instant evaluation of the smoothed Armadillo without iterative re-runs.

\vspace{0.1cm}
Our approach is also faster than NIE, which required 112 min for 100 time steps (67s per step) for the Armadillo case.
\pagebreak

\noindent In contrast, our method trained on the entire interval in 10.1 min using standard initialization. 
In Section~\ref{s-initialization_experiments} (and in the supp. mat.) we give experiments showing that our initialization scheme allows even faster network convergence.

\vspace{0.2cm}

Finally, observe that our smoothing approach gives a \textit{multi-resolution} representation of the initial surface. Thus, future works include incorporating multi-resolution neural networks \cite{lindell2021bacon, paz2023mr} to our framework to fit solutions of the mean curvature equation.

\subsection{Interpolation between implicit surfaces}
\label{s-interpolation_between_surfaces}

Suppose $g_i$ are the SDFs of the Bob and Spot (Fig~\ref{f-spot-bob-interpolation}, left-right) and that $f_\theta$ has $1$ hidden layer $f_i:\R^{128}\to\R^{128}$.
We train $f_\theta$ using $\mathcal{L}_{\text{LSE}}+\mathcal{L}_{\text{data}}$, as described in Sec~\ref{ss-interpolation}.
Line $1$ in Fig~\ref{f-spot-bob-interpolation} shows the reconstructions of the level sets.
\begin{figure}[ht]
    \centering
        \includegraphics[width=0.85\columnwidth]{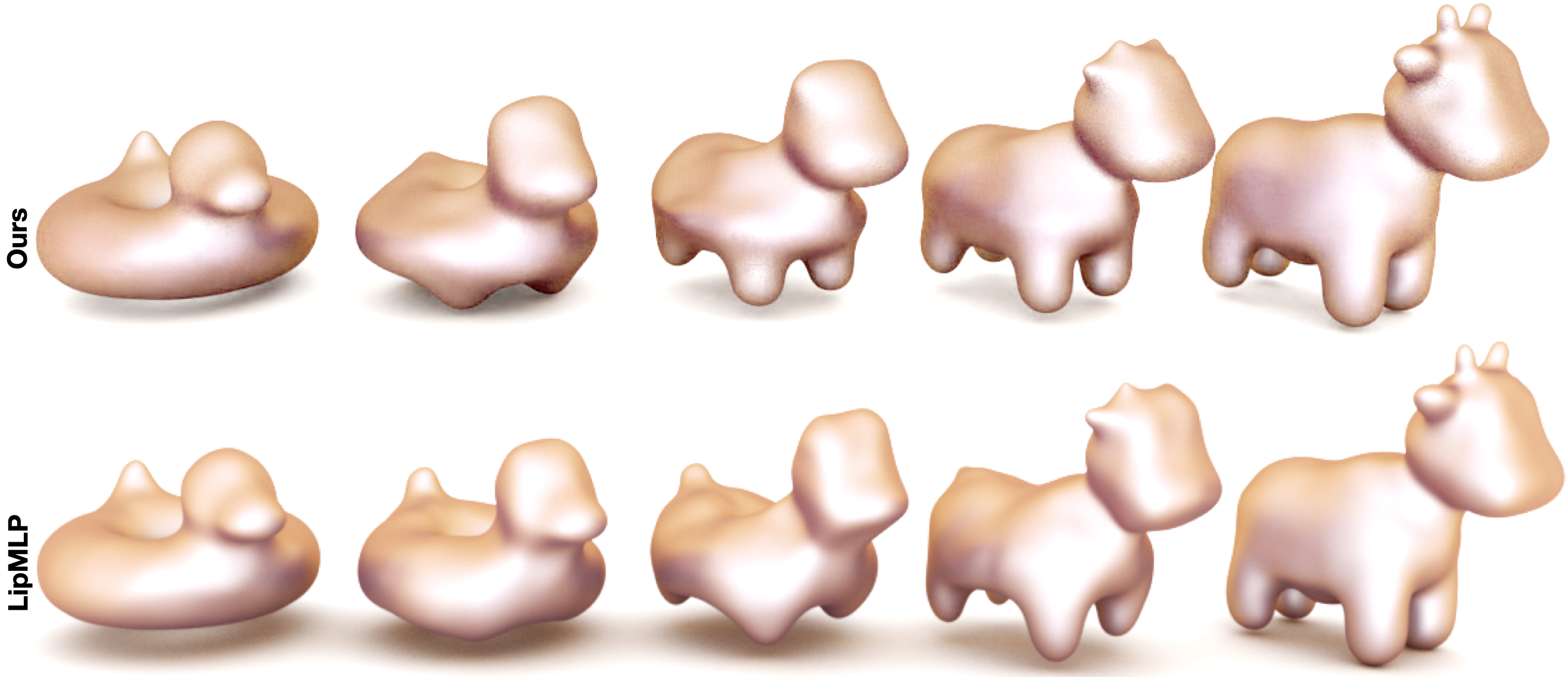}
        \caption{Interpolation between Bob and Spot. Line 1 shows the result using our method, and line 2 the Lipschitz MLP. Notice that our method results in smoother transitions between the images.}
    \label{f-spot-bob-interpolation}
\end{figure}

We compare our method with \textit{Lipschitz MLP}~\cite{liu2022learning} which considers the $tanh$ activation. We use a Lipschitz MLP with $5$ hidden layers of $256$ neurons. Each layer is followed by a Lipschitz regularization. The network was trained during $100000$ epochs. See the resulting interpolation in Line $2$ of Fig~\ref{f-spot-bob-interpolation}.
Our network is significantly smaller than the Lipschitz MLP, but results in natural interpolation. This is due to the high representation capacity of sinusoidal~MLPs.

\begin{figure}[ht]
    \centering
        \includegraphics[width=0.85\columnwidth]{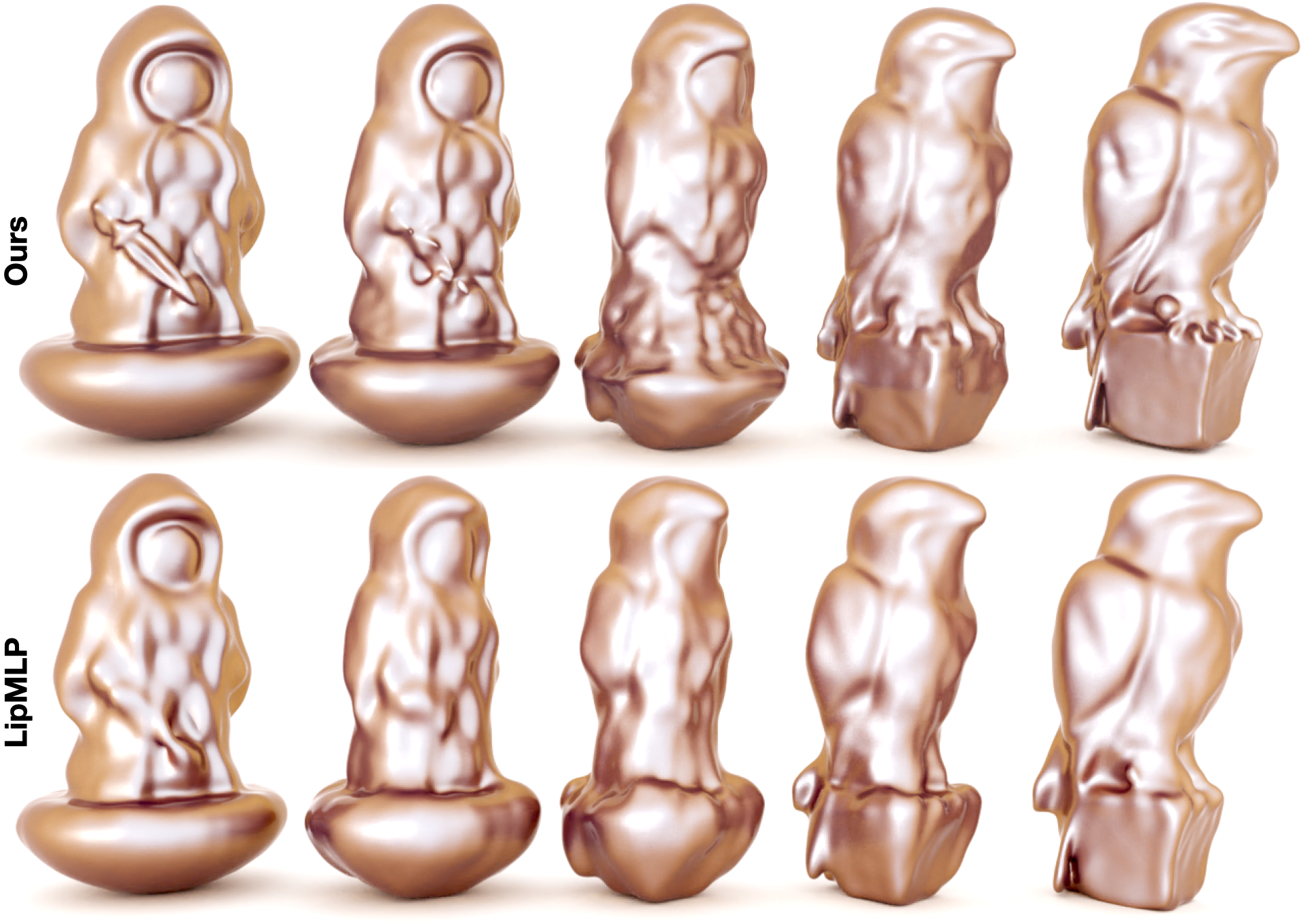}
    \caption{Interpolation between Witch and Falcon. Line 1 (20000 epochs) is the result of our method and Line 2 (100000 epochs) considers the Lipschitz MLP network.
    Notice that our approach results in a better approximation of the initial conditions, as can be seen in the Witch's hood and sword, and Falcon's beak and talons.
    }
    \label{f-witch-falcon-interpolation}
\end{figure}
Fig~\ref{f-witch-falcon-interpolation} shows the reconstructions of interpolation between the Witch and Falcon (from the Thingi10K dataset~\cite{Thingi10K}). The first line is the result of our method using a network with $2$ hidden layers $\R^{128}\!\to\! \R^{128}$ trained during $20000$ epochs. For the Lipschitz MLP, we have to consider a larger network with $5$ hidden layers of $512$ neurons and train it for $100000$ epochs. Even with the added capacity and training iterations, the Lipschitz MLP cannot adequately approximate the initial conditions.

\subsection{Initialization based on trained networks}
\label{s-initialization_experiments}

Let $g_\phi:\R^3\to \R$ be a trained network with $2$ hidden layers $g_i:\R^{256}\to \R^{256}$ that fit the SDF of the Bunny.
Let $f_\theta:\R^3\times\R\to \R$
be a network with $2$ hidden layers $f_i:\R^{256}\to \R^{256}$.
We train $f_\theta$ to approximate a solution of the mean curvature equation subject to
$f_\theta =g_\phi$ on $\R^3\times \{0\}$.

Here we use the scheme of Sec~\ref{s-network-initialization} to define $\theta$ in terms of $\phi$ such that $f_\theta(p,t)=g_\phi(p)$. We compare it with the standard initialization of sinusoidal MLPs~\cite{sitzmann2020implicit}.
\begin{figure}[ht]
    \centering
        \includegraphics[width=\columnwidth]{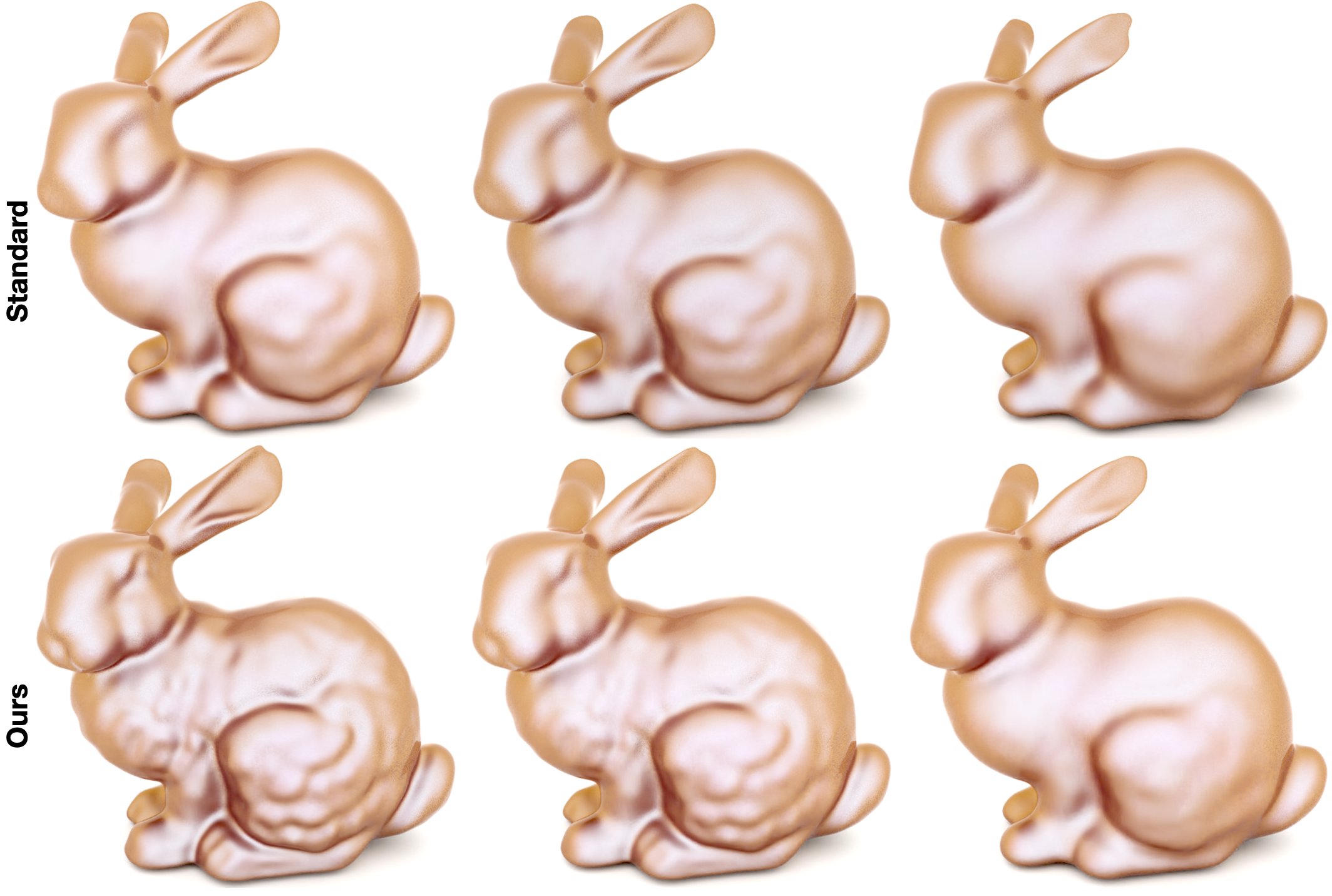}
        \caption{Network initialization comparison. Line 1 shows the results of using the initialization is given in~\cite{sitzmann2020implicit} and Line 2 considers the proposed initialization. Notice how the new network initialization results in a model that can represent higher frequencies, as shown by the increased surface details in Line 2.}
    \label{f-initialization}
\end{figure}

Fig~\ref{f-initialization} gives qualitative comparisons between the initializations, after training $f_\theta$ during $500$ epochs.
We empirically observed that our approach speeds up learning.
For example, Line~1 shows the bunnies at $t=-0.2,0.0,0.6$.
Line~2 gives the analogous results using our initialization, which results in faster convergence. Note the preservation of surface details using the proposed initialization at $t=-0.2, 0.0$ (Line~2), compared to the standard initialization in Line~1.
Fig~\ref{f-initialization_loss} shows the plots of the constraints considering $1500$ epochs. Note that, using our scheme, the training starts closer to a minimum of the loss function.
\begin{figure}[ht]
    \centering
        \includegraphics[width=0.95\columnwidth]{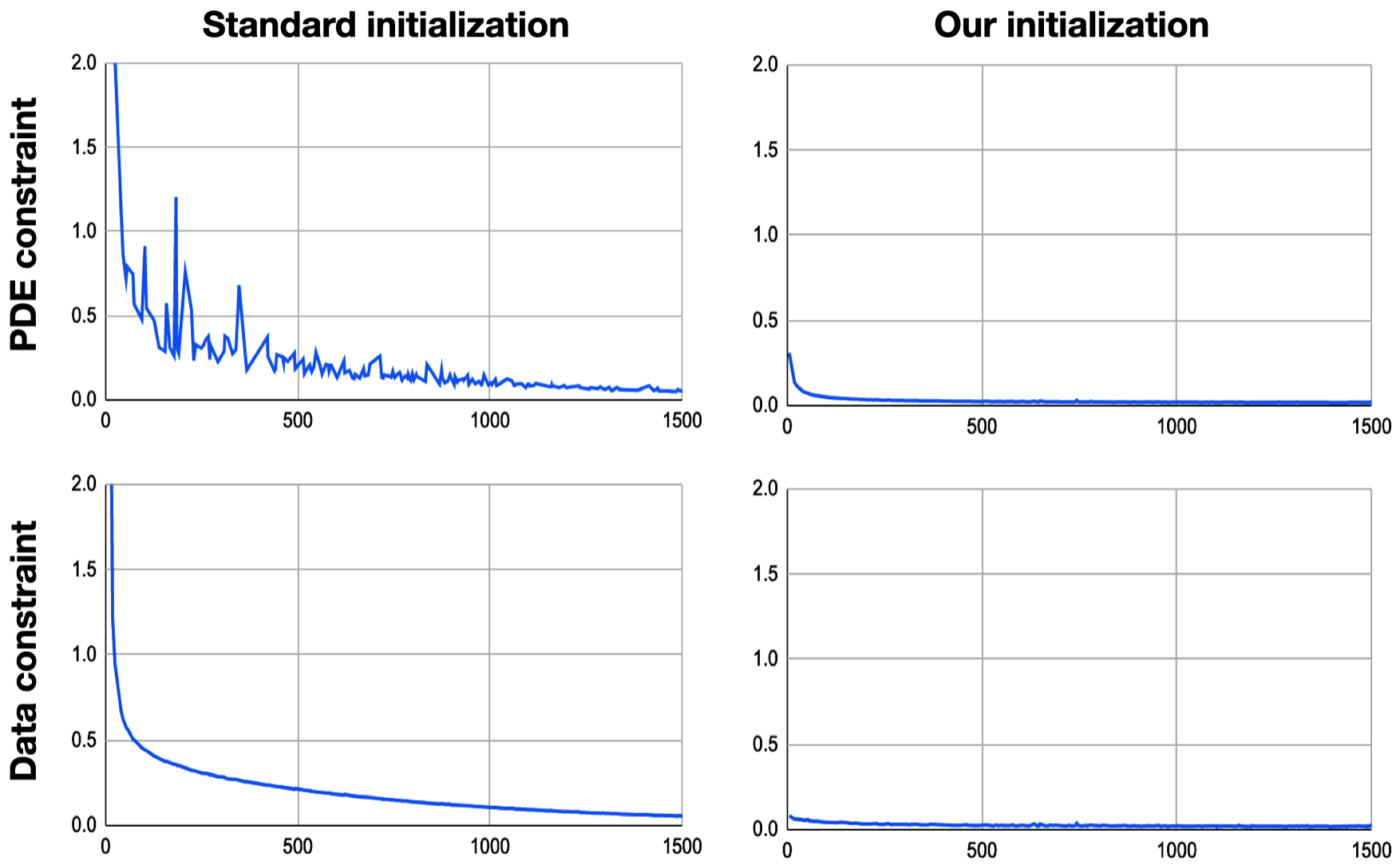}
        \caption{Loss function comparisons. Column 1 shows the plots of the LSE/data constraint using the standard initialization. Column 2 shows the corresponding plots applying our initialization~scheme. The horizontal axis represents the number of epochs.}
    \label{f-initialization_loss}
\end{figure}

\section{Conclusions and Future Work}
We introduced a framework to explore the differentiable properties of smooth networks in the problem of evolving level sets of neural SDFs. For this, we extended their domain to space-time and directly operates on \textit{primitive neural implicits}, which opens up possibilities to control geometric animation and modeling using LSEs, enabling the modeling of multiple surface evolutions in a single method.

The method allows for evolving neural implicit surfaces under LSEs without the use of additional data, only the initial conditions are used.
Note that other methods compute an approximation of the solution using numerical simulations and then fit it into the neural network. However, our framework was able to learn the animation only considering the LSE constraint.
This is powerful because the models are expressed in compact LSEs that are used to define constraints. This approach enables learning the corresponding animations \textit{without any supervision}.
Most techniques in geometry processing use differential equations to model various kinds of phenomena which, in general, are written in terms of their energy formulation.

The resulting networks are smooth approximations of LSE solutions. Traditional numerical solutions are discrete, making non-trivial the task of introducing more conditions. However, this is a quite simple task in our method.
We believe that the development of such methods in graphics would enable the community to use the robustness of classical continuous theories without the need for discretizations.


\section*{Acknowledgments}

\vspace{-0.2cm}

We are grateful to the reviewers for their detailed comments. 
We thank the Stanford Computer Graphics Laboratory for the Bunny and Armadillo models, and Keenan Crane for the Spot and Bob models. The authors thank CNPQ and FAPERJ for financial~support.

{\small
\bibliographystyle{ieee_fullname}
\bibliography{egbib}
}

\end{document}


\title{
Neural Implicit Surface Evolution~\\
-- Supplementary Material --
}

\author{
\normalsize Tiago Novello\\
\small IMPA
\and
\normalsize Vinicius da Silva\\
\small PUC-Rio
\and
\normalsize Guilherme Schardong\\
\small U Coimbra
\and
\normalsize Luiz Schirmer\\
\small Unisinos
\and
\normalsize Helio Lopes\\
\small PUC-Rio
\and
\normalsize Luiz Velho\\
\small IMPA
}

\maketitle
\ificcvfinal\thispagestyle{empty}\fi


\section{Minimal surfaces}
\label{a-minimal-surfaces}
The evolution of a surface $S$ governed by the \textit{mean curvature equation} (MCE) leads to a family of surfaces that reduce their area over time. Let $f\!:\!\mathbb{R}^3\!\times\! \mathbb{R} \!\to\! \mathbb{R}$ be a solution of MCE and $S_t$ be its corresponding surface evolution.
\begin{align}\label{e-mean_curvature_equation}
\begin{cases}
\displaystyle\frac{\partial f}{\partial t} -\alpha\norm{\nabla f}\,\kappa =0 & \text{ in } \R^3\times (a,b),\\
f =g & \text{ on } \R^3\times \{t=0\}.
\end{cases}
\end{align}
The area of $S_t$ can be measured using
$\text{Area}(S_t)=\int_{S_t} dS_t$, where $dS_t$ is the area form of $S_t$.
It can be proved that the \textit{first variation of area} of the family $S_t$ is given by
\begin{align}
    \frac{d}{dt}\text{Area}(S_t)\Bigr|_{t=0}=-\int\limits_{S_0} \kappa^2dS_0,
\end{align}
The proof can be found in \cite[Sec.~3.5]{manfredo2016differential}, \cite[Cor. 6.2]{martin2014introduction}.
Thus, if the mean curvature satisfies $\kappa\neq 0$, the area of $S_t$ initially decreases because its derivatives are negative at $t=0$.

A surface $S_{t_0}$ is \textit{critical} if $\frac{d}{dt}\text{Area}(S_{t})\Bigr|_{t=t_0}\!\!\!\!\!=0$, that is, if $\kappa$ is constant equal to zero. This surface is called \textit{minimal}. Examples of minimal surfaces include the plane, catenoids, helicoids, Enneper surface, Costa's minimal surfaces, etc.

We can fix a region of the initial surface $S$ in the MCE. If $S$ has a boundary curve, fixing it during the evolution leads to a surface of minimal area. This problem is related to the physical shapes of soap films at equilibrium under the surface
tension~\cite{wang2021computing}.

\section{\large Network initialization}
\label{s-network-initialization_appendix}

\begin{proposition}\label{p-initialization}
Let $g_\phi:\R^3\to \R$ and $f_\theta:\R^3\times \R\to \R$ be networks with depth $d$. If $f_\theta$ is wider than $g_\phi$, we can define $\theta$ in terms of $\phi$ such that $f_\theta(p,t)=g_\phi(p)$ for all~$(p,t)$.
\end{proposition}

\begin{proof}
Recall that $
    g_\phi(p)=B_{d+1}\circ g_{d}\circ \cdots \circ g_{1}(p)+b_{d+1}$,
where $g_{i}(p_i)\!=\!\sin (B_i p_i \!+\! b_i)$ is the $i$-layer, $B_i$ is a matrix in $\R^{N_{i+1}\times {N_i}}$,
and $b_i\in\R^{N_{i+1}}$ is the $i$-bias. Analogously, $f_\theta(p,t)=A_{d+1}\circ f_{d}\circ \cdots \circ f_{1}(p,t)+a_{d+1}$,
with its $i$ layer $f_i:\R^{M_i}\to~\R^{M_{i+1}}$ given by $f_{i}(p_i)=\sin (A_i p_i + a_i)$.

By hypothesis, $f_\theta$, and $g_\phi$ have the same depth $d$, and the width of each layer of $g_\phi$ is less than or equal to the width of the respective layer of $f_\theta$, i.e., $N_i\leq M_i$. Thus, we define the hidden layers of $f_\theta$ using
$A_i=\begin{psmallmatrix}B_i & 0\\0 & 0\end{psmallmatrix}$, $a_i=\begin{psmallmatrix}b_i \\0\end{psmallmatrix}$.
Evaluating $(p,c)\in \R^{N_i}\times \R^{M_i-N_i}$ in $f_i$ results~in
$$
f_i(p,c)= \sin  \left(\begin{psmallmatrix}B_i & 0\\0 & 0\end{psmallmatrix}\begin{psmallmatrix}p \\c\end{psmallmatrix}+\begin{psmallmatrix}b_i \\0\end{psmallmatrix}\right)
=
\begin{psmallmatrix}g_i(p) \\0\end{psmallmatrix}.
$$
Thus, defining $A_1=\begin{psmallmatrix}B_1 & 0\\F_p & F_t\end{psmallmatrix}$, $a_1=\begin{psmallmatrix}b_1 \\0\end{psmallmatrix}$, we obtain the desired result because
$$
f_1(p,c)= \sin  \left(\begin{psmallmatrix}B_1 & 0\\F_p & F_t\end{psmallmatrix}\begin{psmallmatrix}p \\t\end{psmallmatrix}+\begin{psmallmatrix}b_i \\0\end{psmallmatrix}\right)
=
\begin{psmallmatrix}g_1(p) \\F_p c+F_t t\end{psmallmatrix}.
$$
In other words, the neurons $g_i(p)$ of the network $g_\phi$ remain intact along the layers.
\end{proof}

The blocks $F_p$ and $F_t$ project the entry points $(p,t)$ to a dictionary of sine waves, which are not considered in the following layers because they are fed to zero blocks.
However, new hidden weights can activate such features as the training advances. In Sec \ref{s-varying-width}, we present experiments varying the width of $f_\theta$ to explore such initialization in the problem of solving the MCE.

\vspace{0.1cm}

To reproduce Prop~\ref{p-initialization} with $f_\theta$ deeper than $g_\phi$, we must be able to add a hidden layer $f(p_i)=\sin(Ap+a)$ to $f_\theta$ which do not exist in $g_\phi$. Thus, it would be desirable to initialize $f$ as an identity layer. Following the above approach, we could define $A=I$ and $a=0$ obtaining $f(p)=\sin(p)$, however, in general, $\sin(p)\neq p$.
This can be fixed using that $\sin(p)\approx p$ when $\norm{p}$ is close to zero.
Therefore, we define $A=\lambda I$, with $\lambda$ being a small number, and multiply the resulting output of $f$ by $\frac{1}{\lambda}$ to keep it close to $p$.

\section{Extracting a network at a given time instant}
\label{a-extraction}
Here, for a given time instant $t$, we extract the network $g_\phi\!=\!f(\cdot, t)\!:\!\R^3\!\to\! \R$ from a neural network $f_\theta\!:\!\R^3\!\times \!\R\!\to\! \R$. Suppose $f_\theta(p,t)=A_{d+1}\circ f_{d}\circ \cdots \circ f_{1}(p,t)+a_{d+1}$,
with each $i$-layer defined by $f_{i}(p_i)=\sin(A_ip_i + a_i)$.

To define $g_\phi$ such that $g_\phi(p)\!=\!f_\theta(p, t)$ for all $(p,t)$, we modify the first layer $f_1(p, t)\!=\!\sin(A_1(p,t) \!+\! a_1)$ of~$f_\theta$. 
\pagebreak

\noindent The~matrix $A_1$ has $4$ column vectors $\{w_1,w_2,w_3,u\}$ in $\R^{M_2}$, where ${M_2}$ is the dimension of the codomain of $f_1$. Denoting $p$ by $(x,y,z)$, we obtain
$$A_1(p,t)=x\cdot w_1+y\cdot w_2+z\cdot w_3+ t\cdot u,$$
We use the matrix $B_1$ consisting of the columns $w_1,w_2,w_3$, and the bias $b_1=a_1+t\cdot u$ to set the first layer $g_1$ of $g_\phi$.
Specifically, we define $ g_\phi $ through:
$$
    g_\phi(p)=A_{d+1}\circ f_{d}\circ \cdots \circ f_2 \circ g_{1}(p)+a_{d+1}.
$$
Note that $g_\phi$ equals $f_\theta$, except for its first layer $f_{1}(p,t)$, which is  replaced by $g_1(p)$. We define it as
\begin{align*}
    g_1(p)&=\sin\Big(\underbrace{x\cdot w_1+y\cdot w_2+z\cdot w_3}_{B_1p}+\underbrace{t\cdot u + a_1}_{b_1}\Big).
\end{align*}
From the definition of $g_\phi$, we have $g_\phi(p)=f_\theta(p,t)$, which implies a kind of opposite direction of Prop~\ref{p-initialization}.

\begin{proposition}
Let $f_\theta\!:\!\R^3\!\!\times \!\R\!\to\! \R$ be a neural network, and $t\!\in\!\R$. There is a network $g_\phi\!:\!\R^3\!\to\! \R$ with the same hidden layers of $f_\theta$ such that $f_\theta(p,t)\!=\!g_\phi(p)$ for all $p\!\in\!\R^3$.
\end{proposition}

\section{Ablation studies}
The ablation studies detailed below were performed using the MCE (Eq \ref{e-mean_curvature_equation}) under two settings: With our initialization scheme, presented in Prop~\ref{p-initialization}, and the standard initialization~\cite{sitzmann2020implicit}. The goal is to compare the training convergence of the \textit{data constraint} $\mathcal{L_{\text{data}}}$ and the \textit{LSE constraint} $\mathcal{L_{\text{LSE}}}$ for both initialization schemes under different circumstances.

We will visualize the graphs of $\mathcal{L_{\text{data}}}$ and $\mathcal{L_{\text{LSE}}}$ during the training of a neural network $f_\theta:\R^3\times \R\to \R$ to satisfy the MCE within a time interval $(a,b)$.
In the upcoming experiments, we will use the SDF $g:\R^3\to \R$ of the Bunny as the initial condition,  $f=g$ on $\R^3\times \{0\}$. To initialize $f_\theta$ using our method, we approximate $g$ by a network $g_\phi$.

\subsection{Varying time interval}
We vary $(a,b)$ in the MCE with scale $\alpha\!=\!0.001$.
We recall that the initial condition is at $0\!\!\in\!\! (a,b)$ and on the positive (negative) part, the MCE smooths (sharpens)~it.
Thus, the positive part should be easier to train since no higher frequencies would arise. In contrast, training the negative part creates new higher frequencies, which could take longer to learn.
We evaluate it in the following intervals:
\begin{align*}
    (a,b)=& (0,0.25), (0,0.5), (0,1),\\
    &(-0.1,0.25), (-0.1,0.5), (-0.1,1),\\
    &(-0.25,0.25), (-0.25,0.5), (-0.25,1).
\end{align*}
Fig \ref{fig:data_constraint_times} and \ref{fig:lse_constraint_times} present the data and LSE constraint convergences for these intervals.
As expected, our initialization (top image) provides a better training convergence. For~$\mathcal{L_{\text{data}}}$, this is due to the fact that $f_\theta=g_\phi$ at $t=0$, thus, $\mathcal{L_{\text{data}}}$ only have to maintain this restriction.
\begin{figure}[H]
    \centering
    \includegraphics[width=\columnwidth]{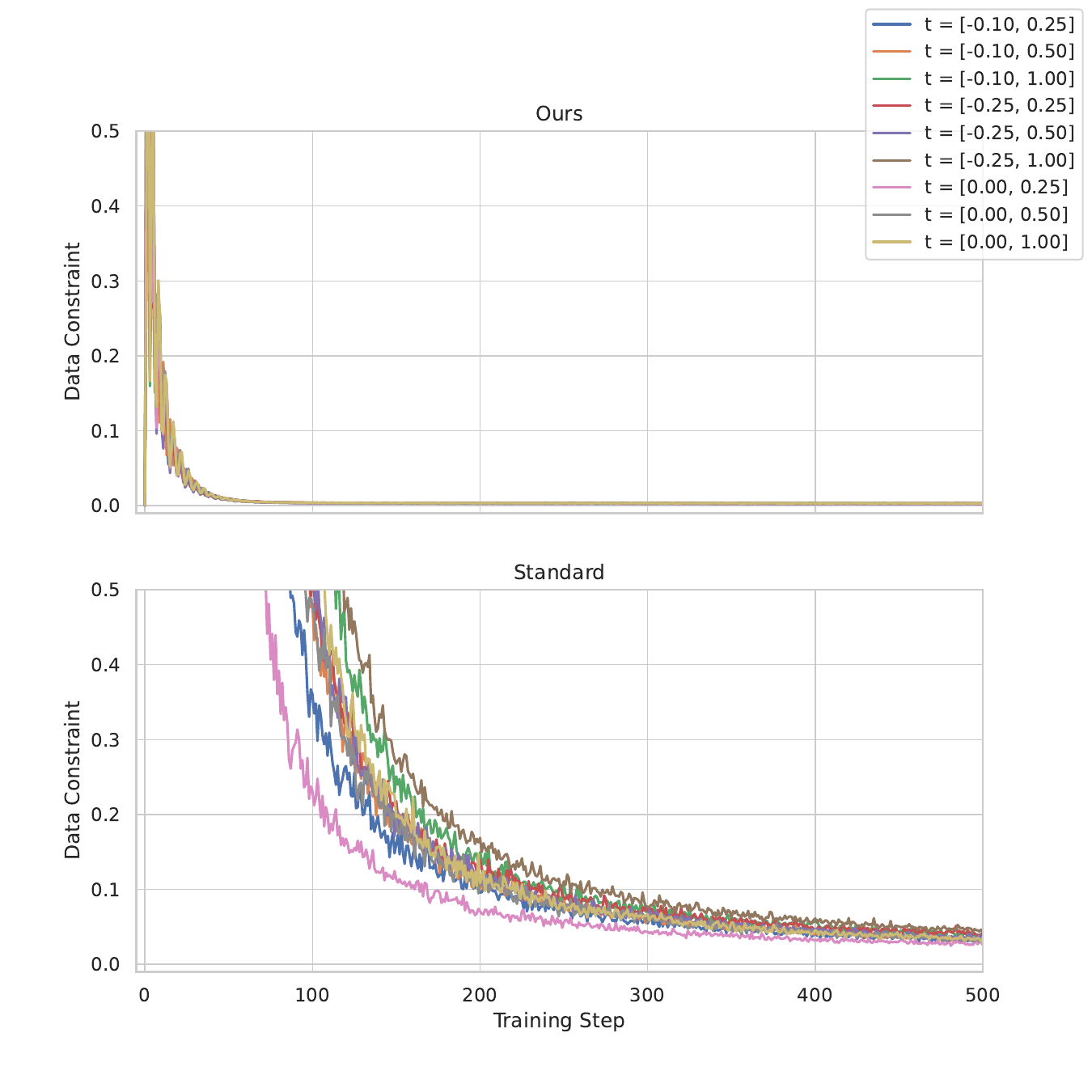}
    \vspace{-0.8cm}
    \caption{Data constraint values for different training intervals.}
    \label{fig:data_constraint_times}
\end{figure}

The convergence of the constraints $\mathcal{L}_{\text{data}}$ and $\mathcal{L}_{\text{LSE}}$ is faster for both initializations when using smaller intervals, as can be seen in the case of $(0,0.25)$ (in purple).
They also take longer to train on intervals with a negative part. This is likely because the solutions in such regions are sharper, requiring, thus, more frequencies for accurate representation, if a solution exists at all.
\begin{figure}[H]
    \centering
    \includegraphics[width=\columnwidth]{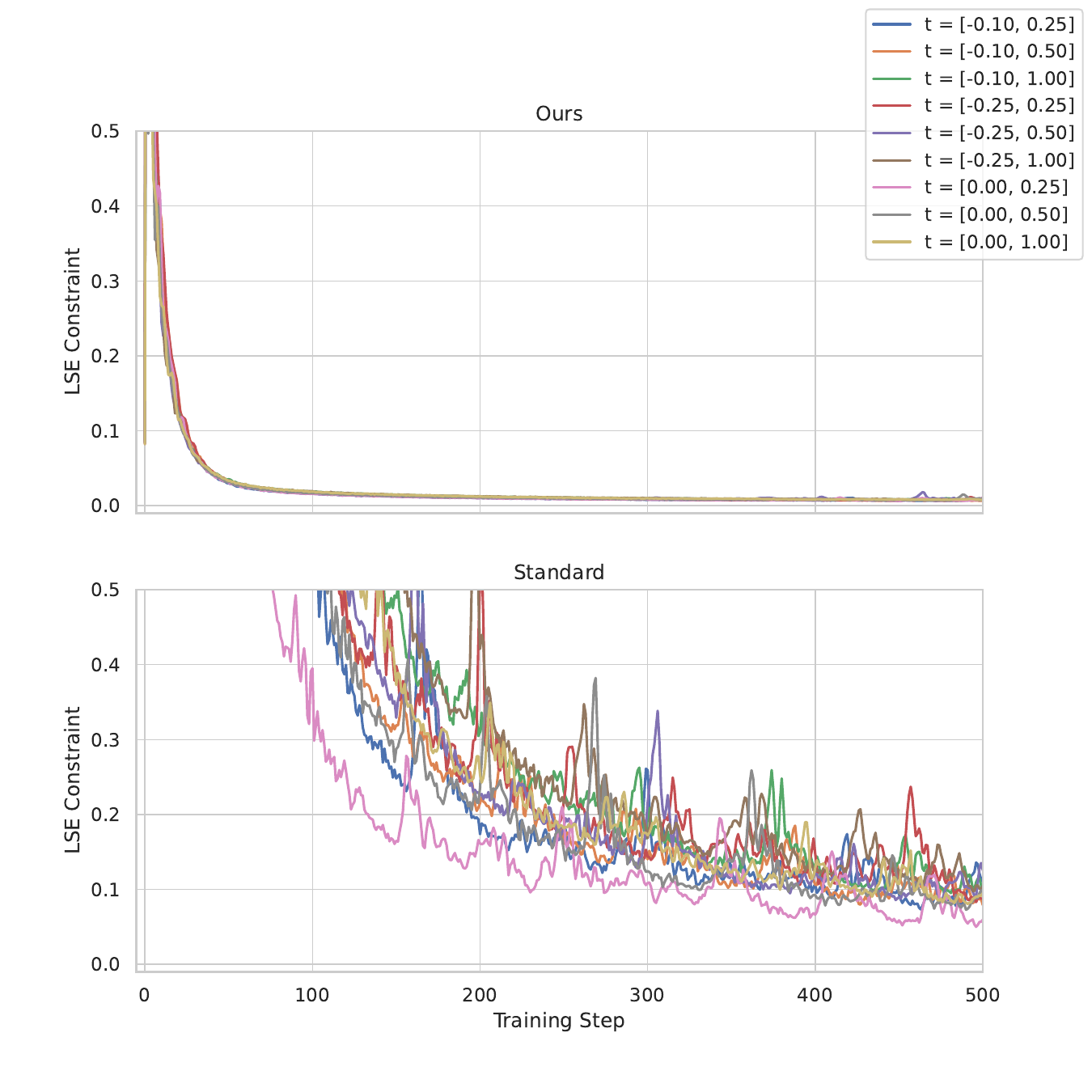}
    \vspace{-0.8cm}
    \caption{LSE constraint values for different training intervals.}
    \label{fig:lse_constraint_times}
\end{figure}

\subsection{Varying MCE scale}
We use the interval $(a,b)=(-0.1, 1)$ and vary the scales $\alpha = i \times 10^{-3}$ for $i = 1,2,3,4,5,10,100$. In theory, increasing $(a,b)$ while fixing $\alpha$ is equivalent to the previous experiment. However, in practice, the representation capacity of $f_\theta$ may not be enough to learn large variations in a short time period. This is evident in Figs~\ref{fig:data_constraint_scales}-\ref{fig:lse_constraint_scales}, where the convergence of $\mathcal{L}_{\text{data}}$ and $\mathcal{L}_{\text{LSE}}$ is sorted by $\alpha$. In general, our initialization results in a better convergence, but we observed that when using a high scale $\mathcal{L}_{\text{data}}$  diverges first, since $\mathcal{L}_{\text{LSE}}$ dominates the training. See the case $\alpha=0.1$ (in purple).
\vspace{-0.4cm}
\begin{figure}[H]
    \centering
    \includegraphics[width=\columnwidth]{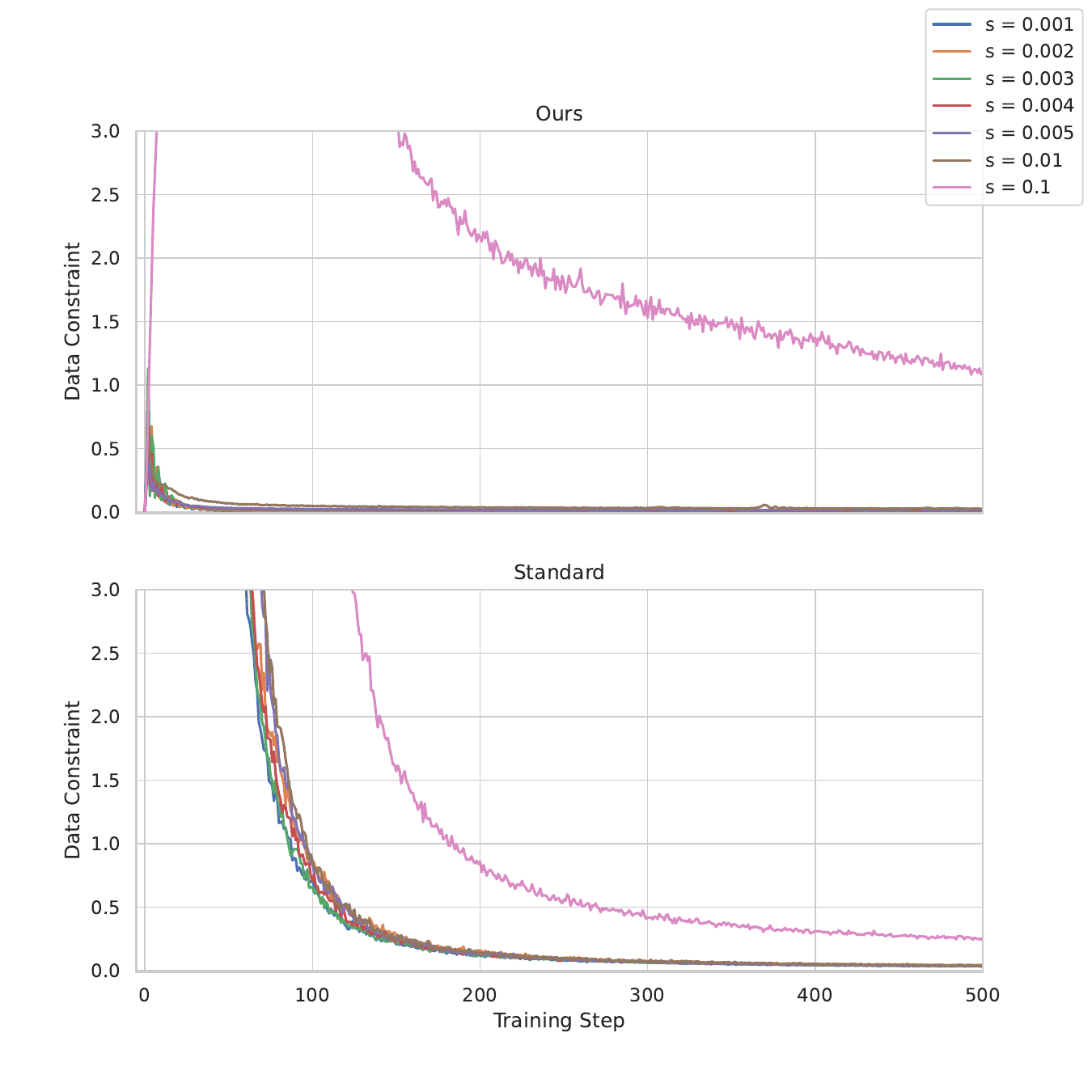}
    \vspace{-0.8cm}
    \caption{Data constraint values for different MCE scale values.}
    \label{fig:data_constraint_scales}
\end{figure}

\vspace{-0.4cm}
\begin{figure}[H]
    \centering
    \includegraphics[width=\columnwidth]{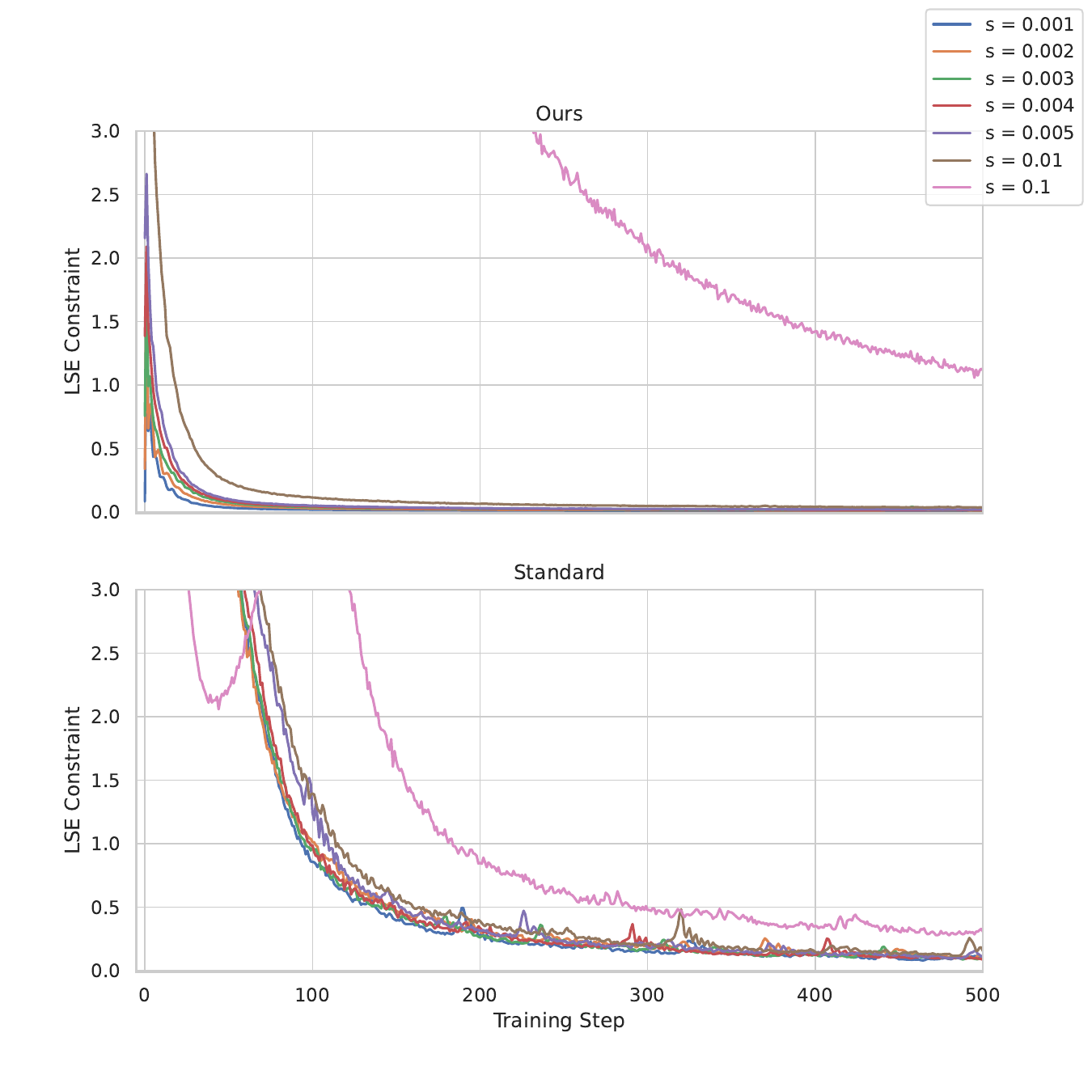}
    \vspace{-0.8cm}
    \caption{LSE constraint values for different MCE scale values.}
    \label{fig:lse_constraint_scales}
\end{figure}

\subsection{Varying the point-sampling proportions}


This experiment aimed to evaluate how the point-sampling of initial and intermediate conditions impacts the training convergence of $\mathcal{L}_{\text{data}}$ and $\mathcal{L}_{\text{LSE}}$. We used the default point-sampling proportions of $\{l_1,l_2,l_3\}=\{0.25,0.25,0.5\}$, as well as $\{0.1,0.1,0.8\}$ and $\{0.4,0.4,0.2\}$. Here, $l_1$, $l_2$, and $l_3$ are the numbers of space-time, on-surface, and off-surface points sampled at each training step (see Sec 4.2 of the main paper).

Figs \ref{fig:data_constraint_proportions}-\ref{fig:lse_constraint_proportions} present the convergences of the resulting constraints during training. It can be observed that sampling fewer points at $t=0$ results in a better convergence.

\vspace{-0.4cm}
\begin{figure}[H]
    \centering
    \includegraphics[width=\columnwidth]{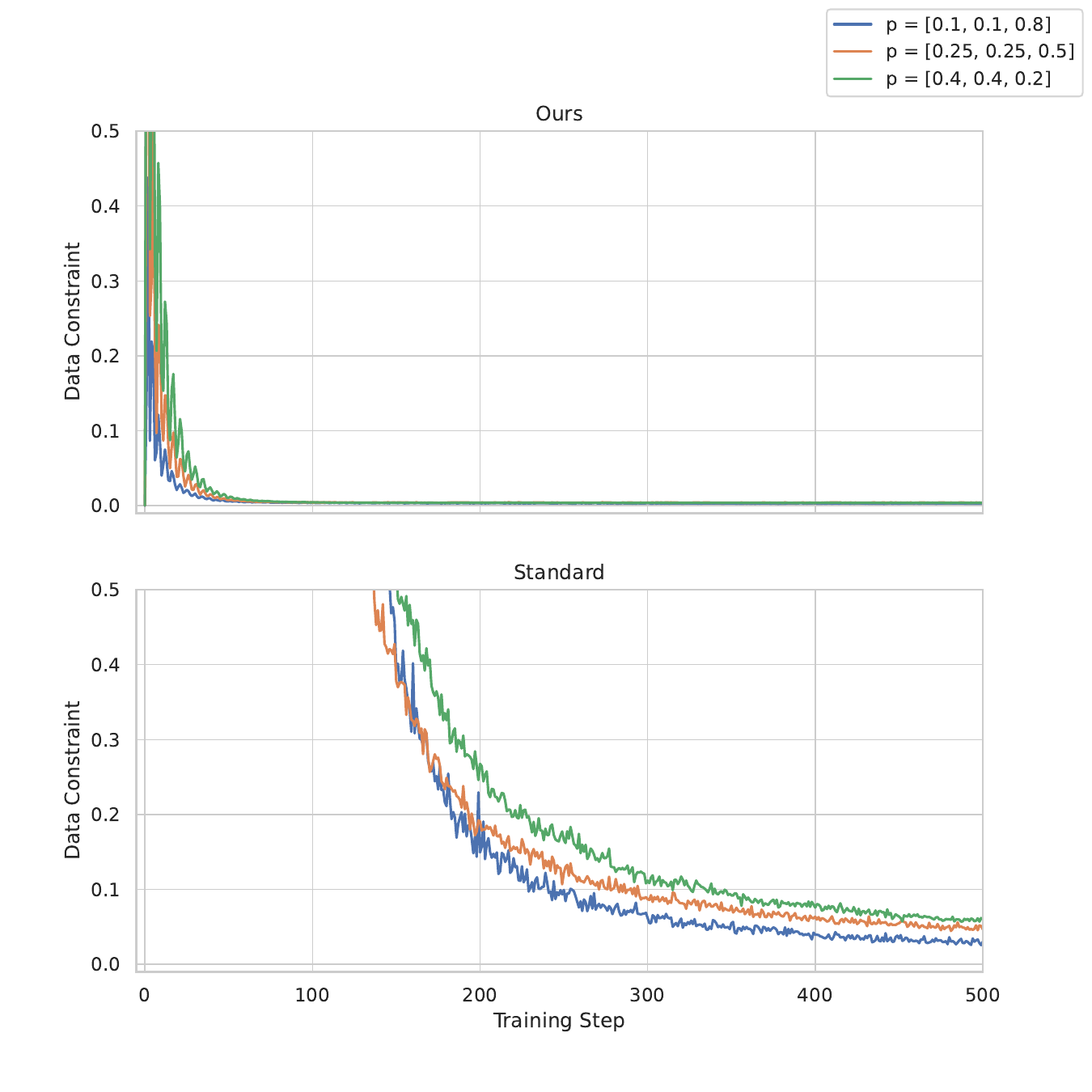}
    \vspace{-0.8cm}
    \caption{$\mathcal{L}_{\text{data}}$ values for different sampling proportions.}
    \label{fig:data_constraint_proportions}
\end{figure}

\vspace{-0.4cm}
\begin{figure}[H]
    \centering
    \includegraphics[width=\columnwidth]{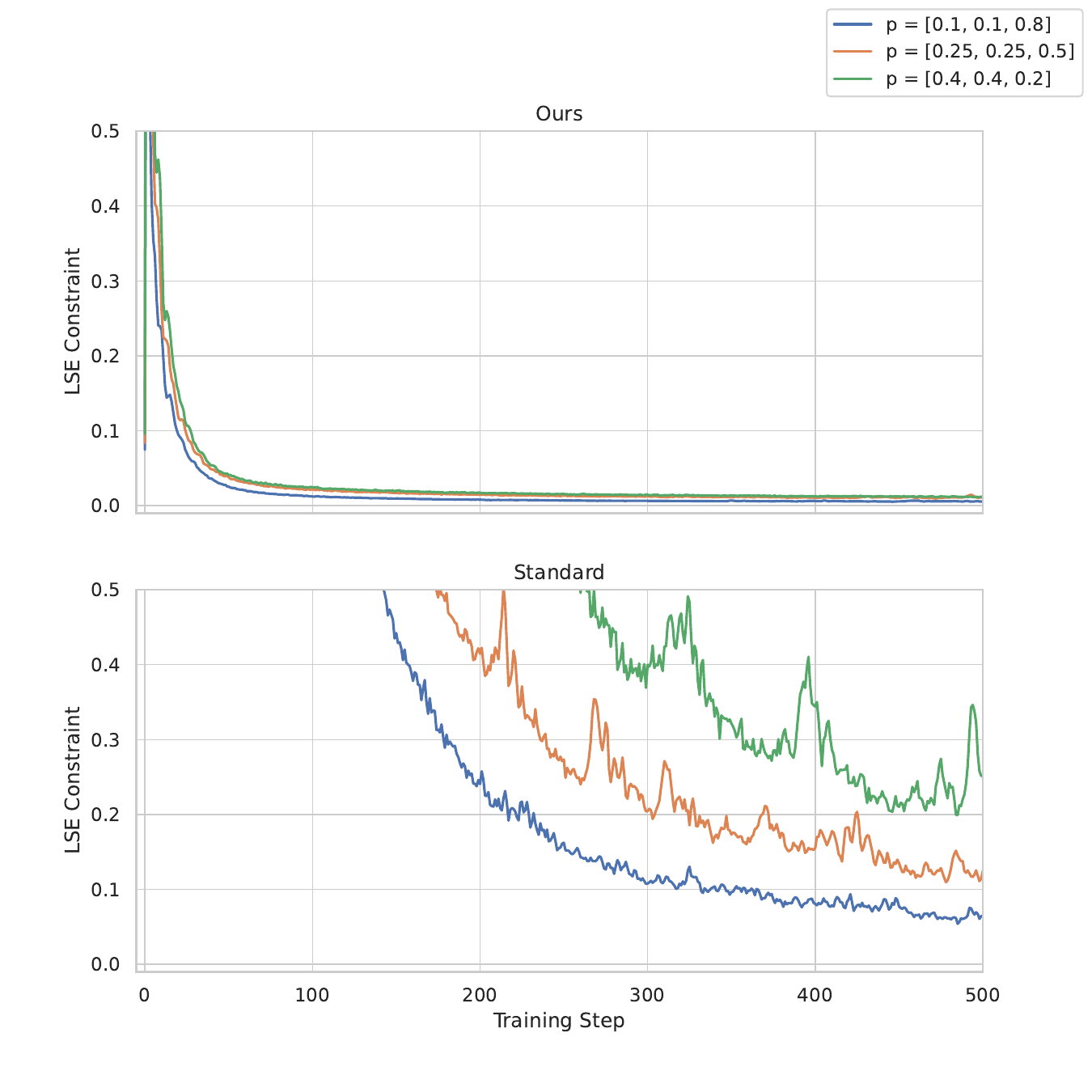}
    \vspace{-0.8cm}
    \caption{$\mathcal{L}_{\text{LSE}}$ values for different sampling proportions.}
    \label{fig:lse_constraint_proportions}
\end{figure}

However, when using different sampling proportions, we obtain new constraints $\mathcal{L}_{\text{data}}$ and $\mathcal{L}_{\text{LSE}}$. Also, sampling fewer points at $t=0$ can result in a longer convergence time for $\mathcal{L}_{\text{data}}$, as shown in the initial condition (Bunny) for each proportion in Fig \ref{fig:my_label}. This is probably due to the \textit{spectral bias} phenomenon: lower frequencies are learned first. As a result, $\mathcal{L}_{\text{LSE}}$ benefits from having a smoother initial condition and prevents fitting at $t=0$.

\vspace{-0.3cm}
\begin{figure}[H]
    \centering
    \includegraphics[width=\columnwidth]{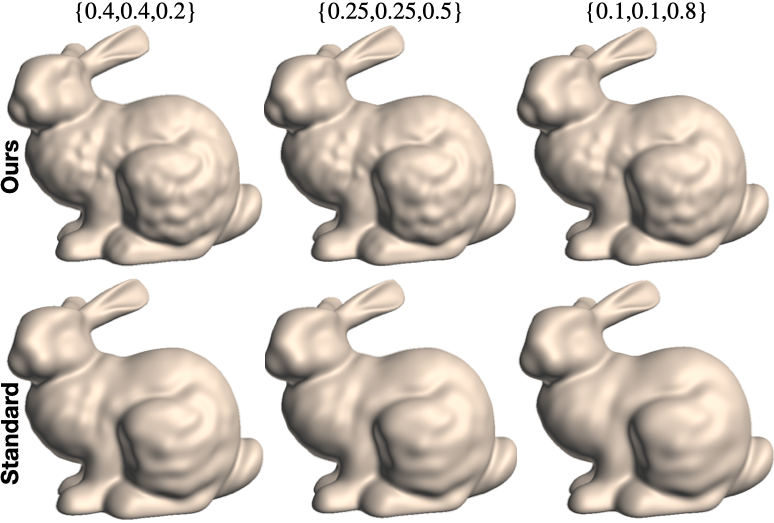}
    \vspace{-0.5cm}
    \caption{The zero-level sets of $f_\theta$ at $t=0$ trained using the proportions $\{0.1,0.1,0.8\}$, $\{0.25,0.25,0.5\}$, and $\{0.4,0.4,0.2\}$.}
    \label{fig:my_label}
\end{figure}

\vspace{-0.6cm}
\subsection{Varying the network width}\label{s-varying-width}


This experiment evaluates the impact of the network width on the training convergence using our standard initializations. We began with a width of 128 neurons and increased it by 16 neurons to a limit of 256. The remaining parameters are set to $(a,b)\!=\!(-0.25,1)$, $\alpha\!=\!1e\!-\!3$, $\{l_1,l_2,l_3\}\!=\!\{0.25,0.25,0.5\}$. As expected, increasing the width leads to better convergence; see Figs \ref{fig:data_constraint_netwidth}-\ref{fig:lse_constraint_netwidth}.

\vspace{-0.4cm}
\begin{figure}[H]
    \centering
    \includegraphics[width=\columnwidth]{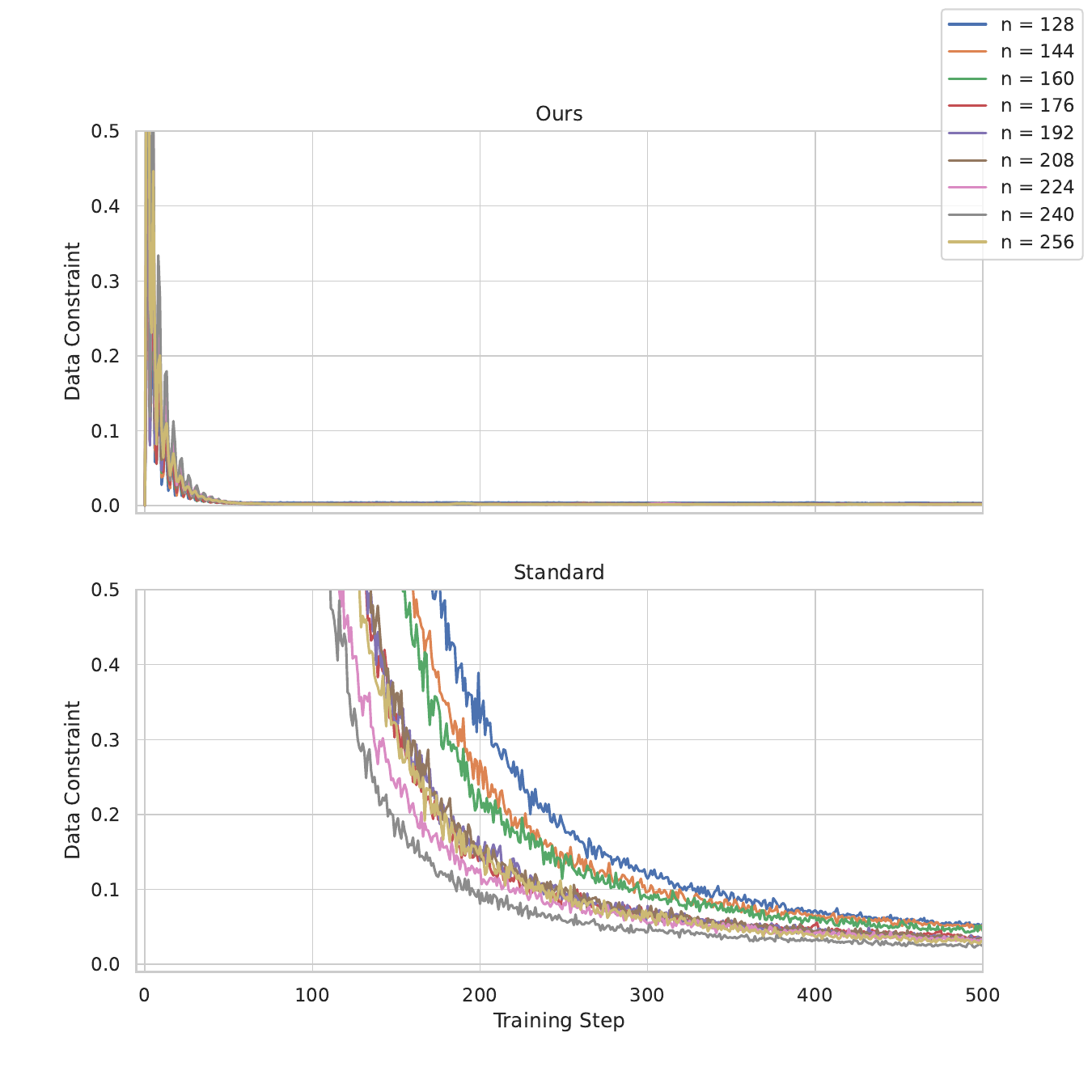}
    \vspace{-0.9cm}
    \caption{Data constraint values for different network widths.}
    \label{fig:data_constraint_netwidth}
\end{figure}

\begin{figure}[ht]
    \centering
    \includegraphics[width=\columnwidth]{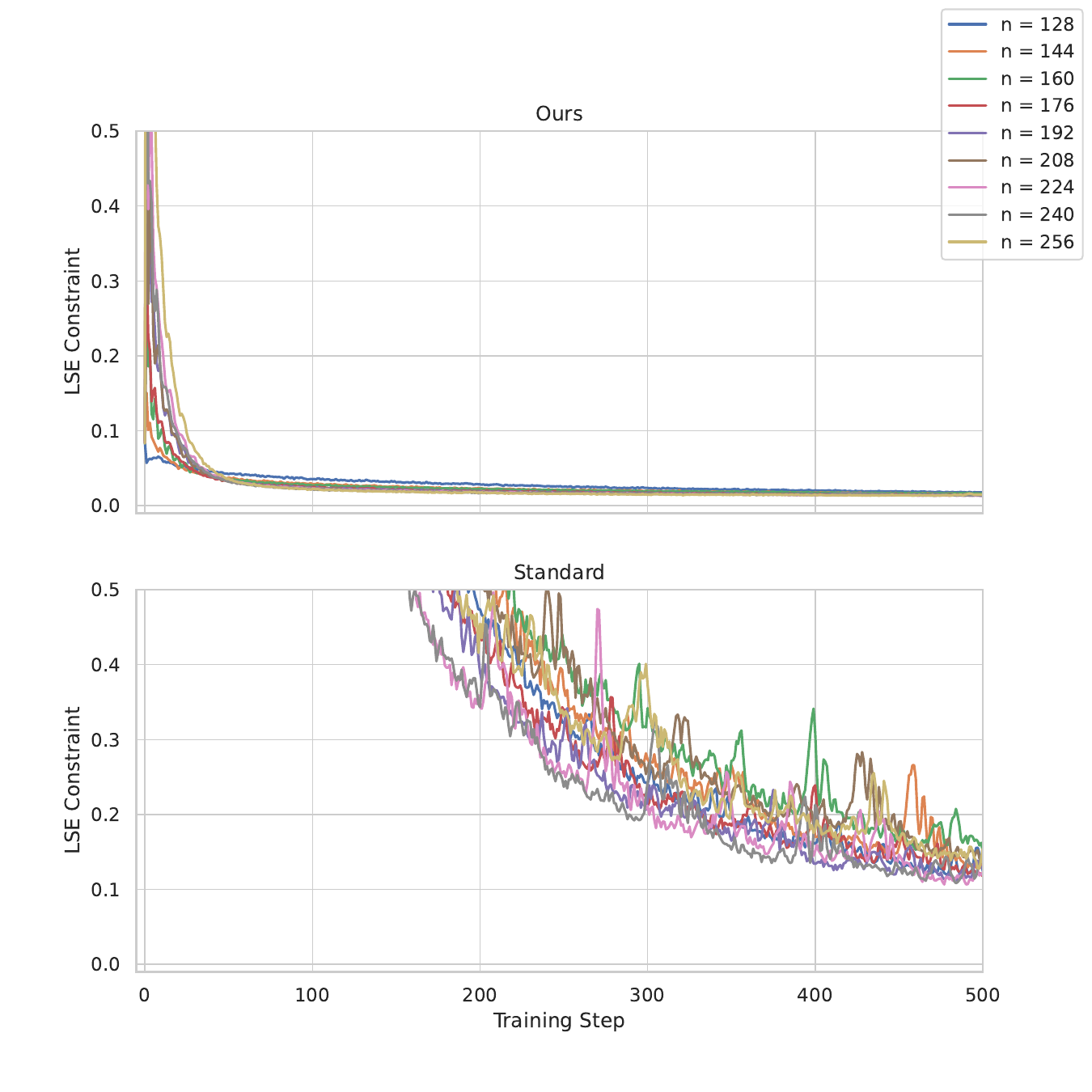}
    \vspace{-0.8cm}
    \caption{LSE constraint values for different network widths.}
    \label{fig:lse_constraint_netwidth}
\end{figure}

\subsection{Varying the initial condition}\label{s-varying-init}
Finally, we vary the initial condition of the MCE to evaluate the convergence of the network $f_\theta$ on different models: Bob, Max Planck, Falcon, Witch, and Neptune.
During the training of each model, we fix the sampling proportions to $\{l_1,l_2,l_3\}\!=\!\{0.25,0.25,0.5\}$. Table~\ref{table} presents the network architectures (the width of the hidden layers), the time required for training 500 epochs, the animation interval $(a,b)$, and the MCE scale $\alpha$ parameters.

\vspace{-0.2cm}
\begin{table}[ht]
\small
\begin{tabular}{l|l|c|l|l}
\hline
\textbf{Model}    & \textbf{Network arch.}        &\textbf{Interval}   &\textbf{Scale}& \textbf{Time (s)} \\ \hline
Bob      & ${[}64,64{]}$       &$(-0.5,1)$&  $1e-2$&    5.04          \\
Max      & ${[}128,128,128{]}$ &$(-0.5,1)$&  $2e-3$&   7.07          \\
Falcon   & ${[}160,160,160{]}$ &$(-0.1,1)$&  $1e-3$&   112.17         \\
Witch    & ${[}256,256,256{]}$ &$(-0.5,1)$&  $1e-3$&   143.34         \\
Neptune  & ${[}300,300,300{]}$ &$(-0.1,1)$&  $2e-4$& 162.42 \\
\hline
\end{tabular}
\caption{The network architectures and the time spent in their training to learn the evolution of the Bob, Max Planck, Falcon, Witch, and Neptune surfaces under the MCE.}
\label{table}
\end{table}

\vspace{-0.3cm}

For the sampling of on-surface points, we used different point clouds sampled from the original models. This affects the time needed to train our networks, as each epoch is defined as a complete iteration over the point-cloud. The Bob, Max Planck, Falcon, Witch, Neptune have 5344, 5002, 72466, 77553, 72668 points, respectively.


Fig \ref{fig:init_conditions} illustrates the zero-level sets of the resulting evolutions of Bob, Max Plank, Falcon, Witch, and Neptune models (middle). The sharpened models are on the left column. Notice that their geometric features are enhanced. Particularly, Max Planck's nose, mouth and ears are noticeably more prominent. The same occurs for the Wizard's sword and cape, and Neptune's hands and spear tip.
\begin{figure}[H]
    \centering
    \includegraphics[width=0.335\columnwidth]{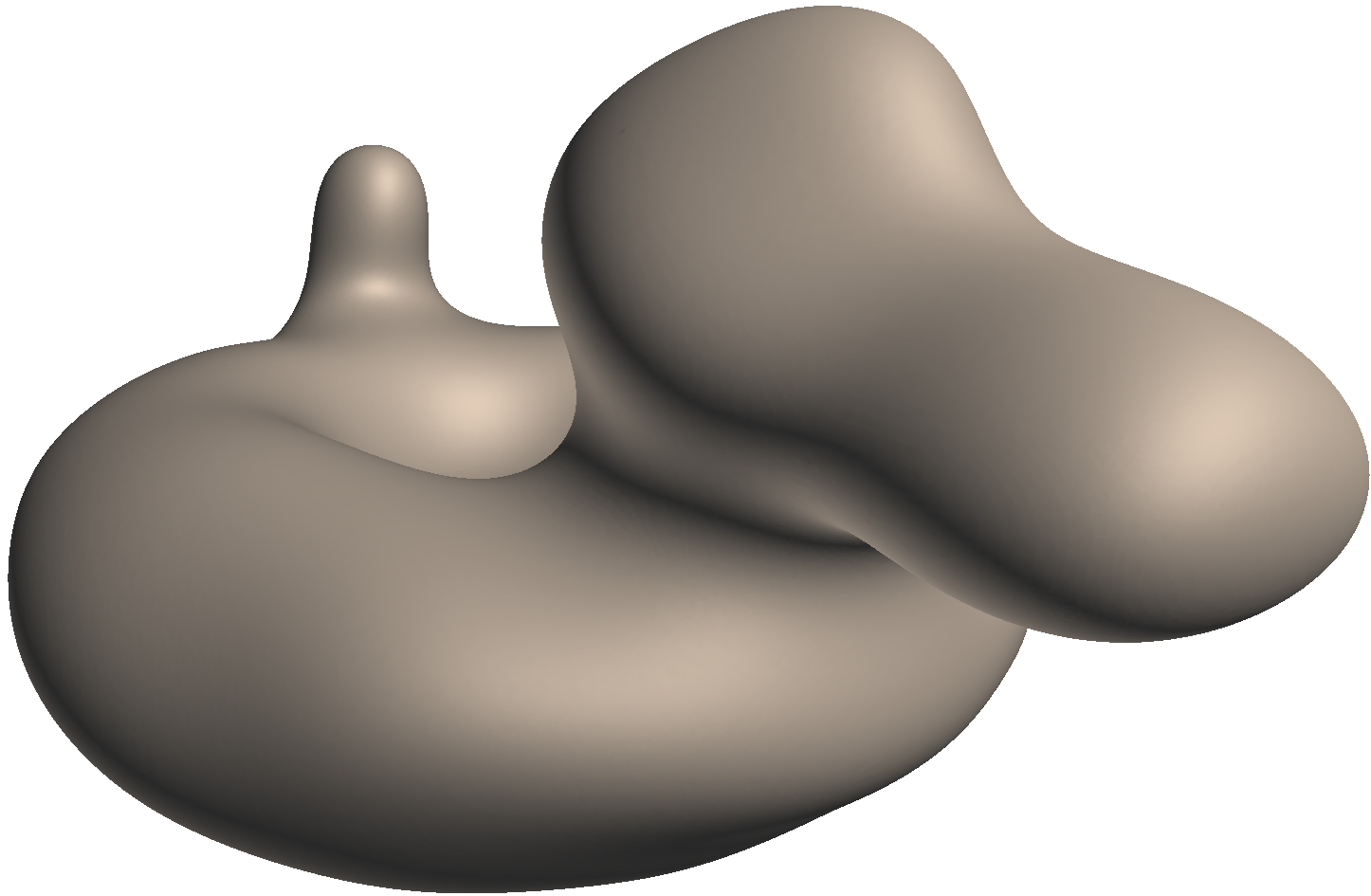}
    \includegraphics[width=0.32\columnwidth]{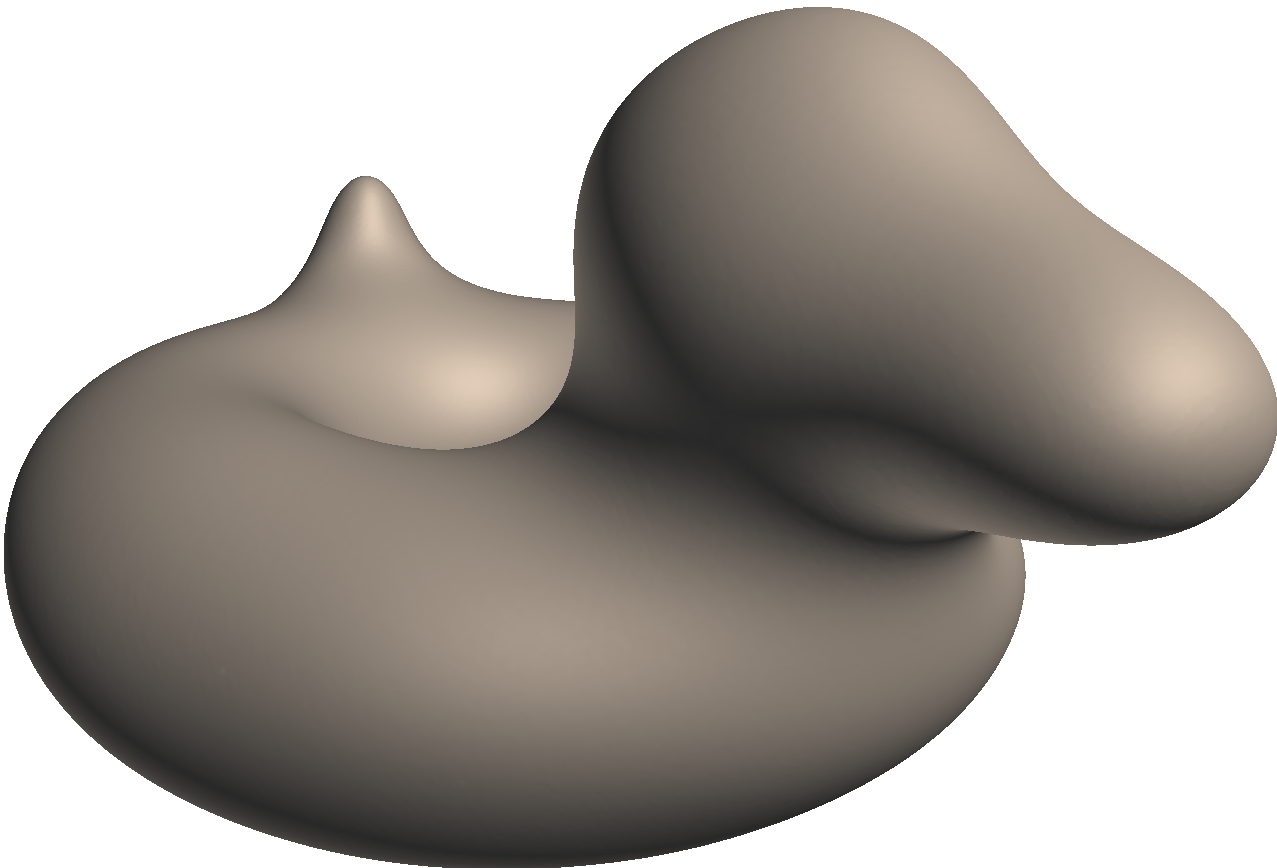}
    \includegraphics[width=0.31\columnwidth]{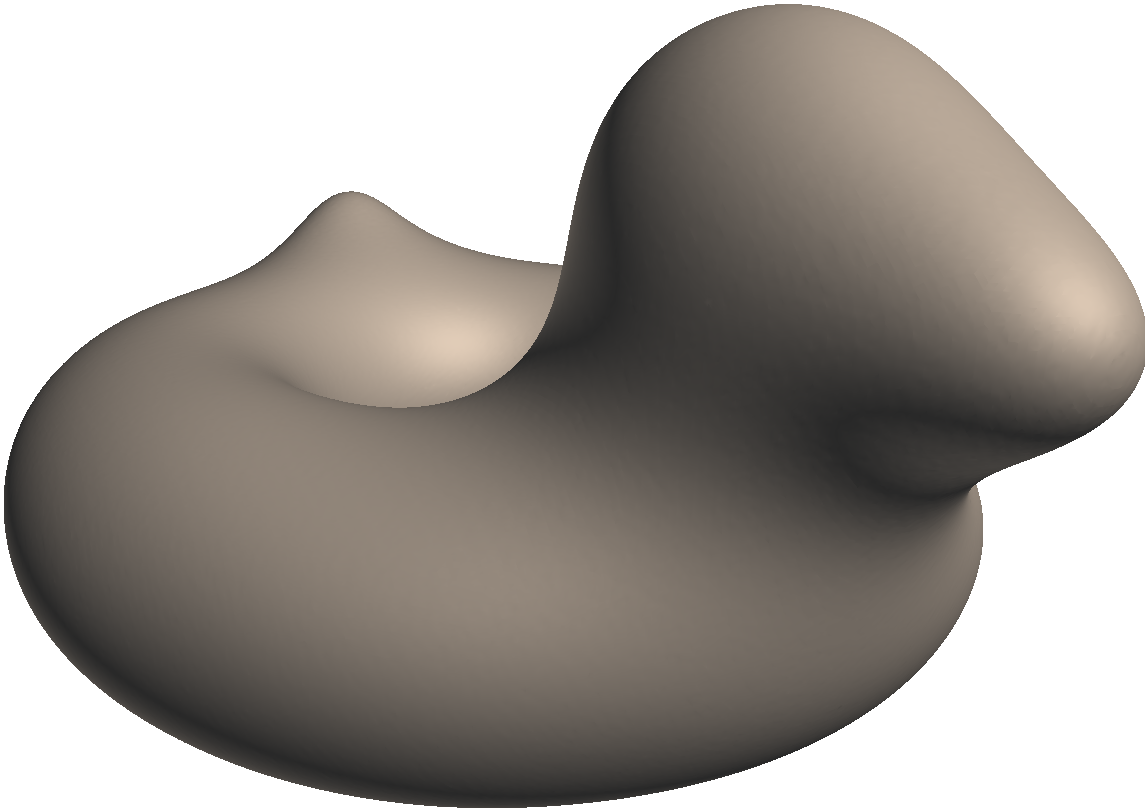}
    \includegraphics[width=0.315\columnwidth]{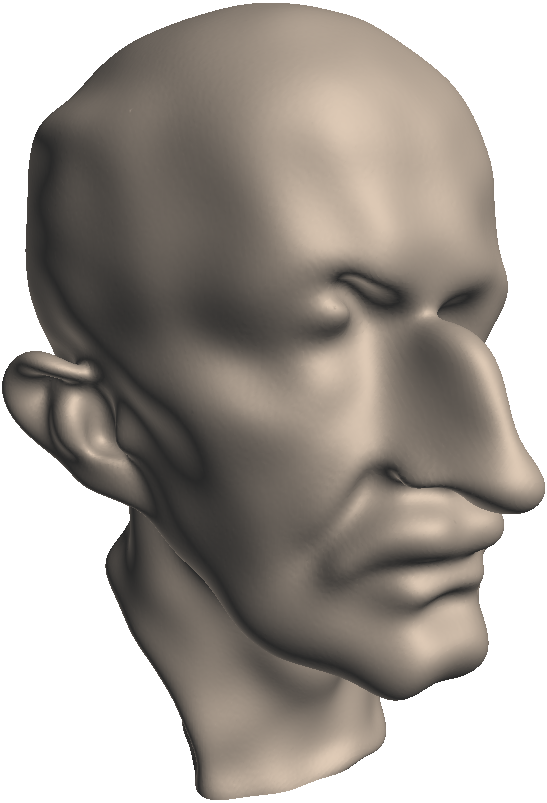}
    \includegraphics[width=0.30\columnwidth]{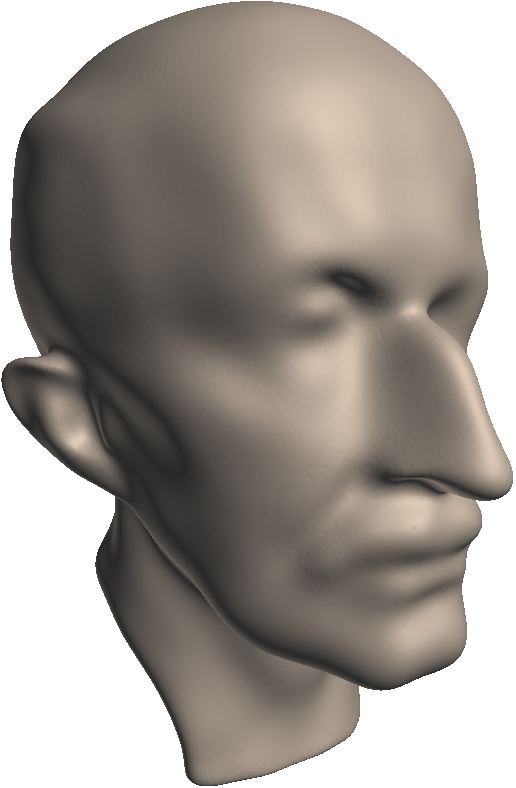}
    \includegraphics[width=0.29\columnwidth]{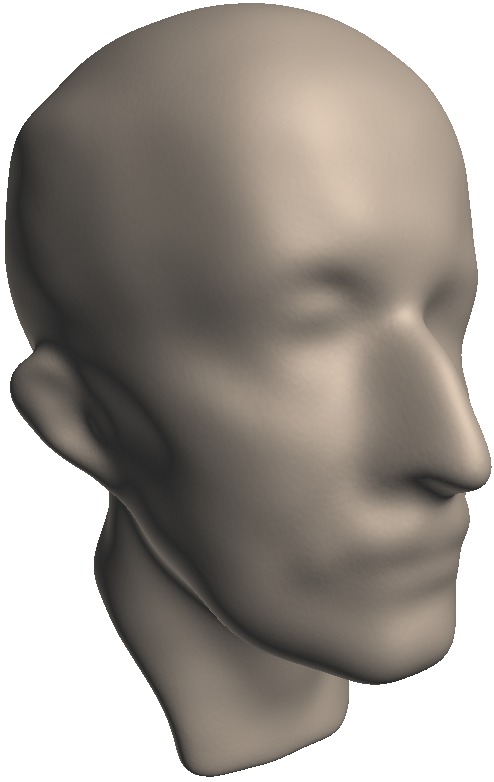}
    \includegraphics[width=0.28\columnwidth]{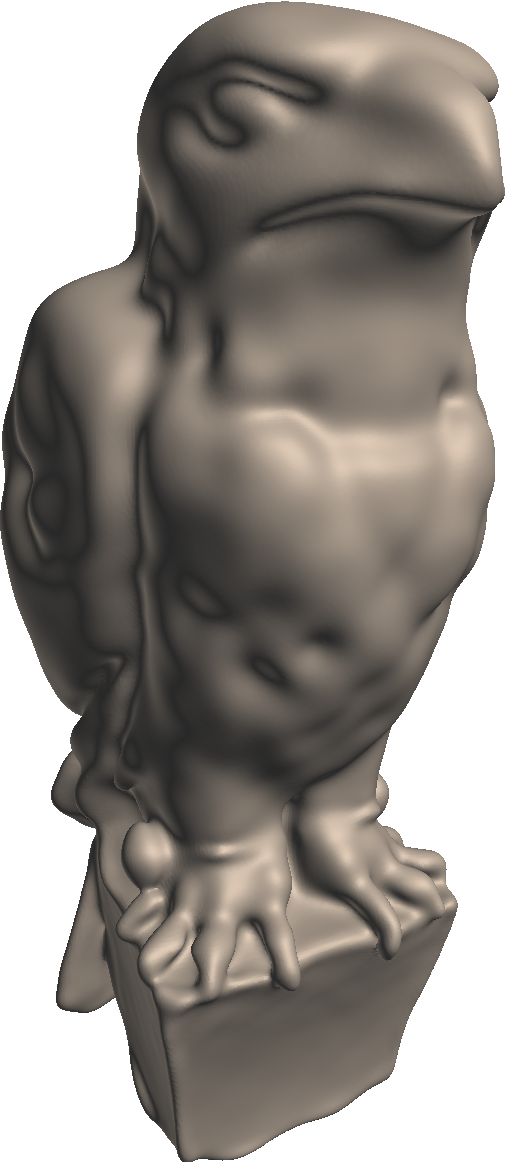}
    \includegraphics[width=0.27\columnwidth]{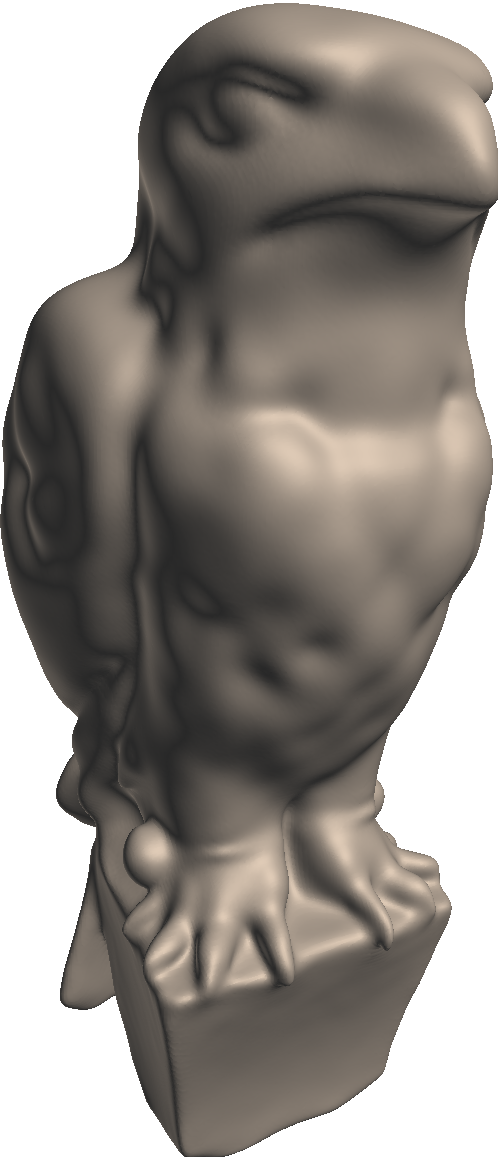}
    \includegraphics[width=0.26\columnwidth]{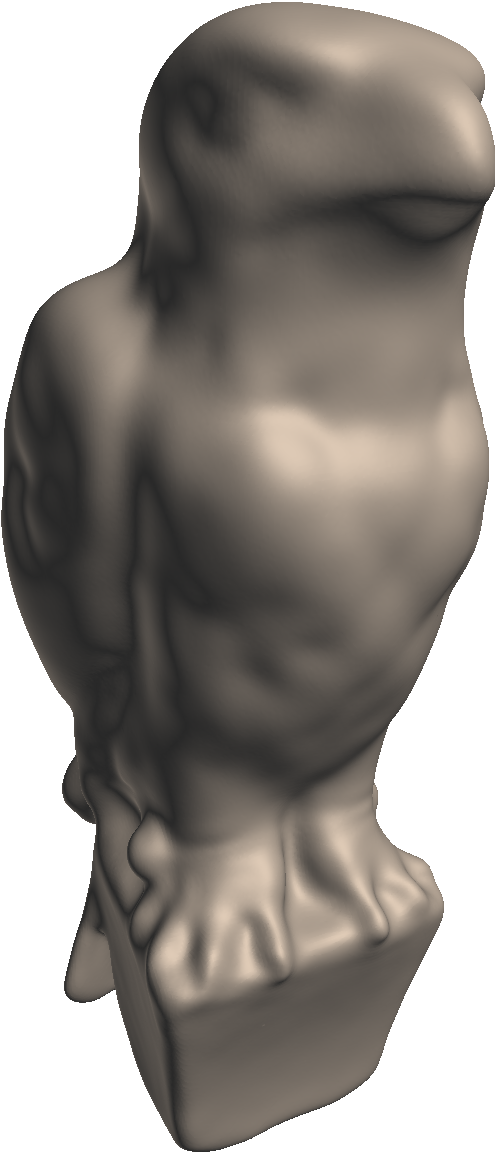}

    \includegraphics[width=0.31\columnwidth]{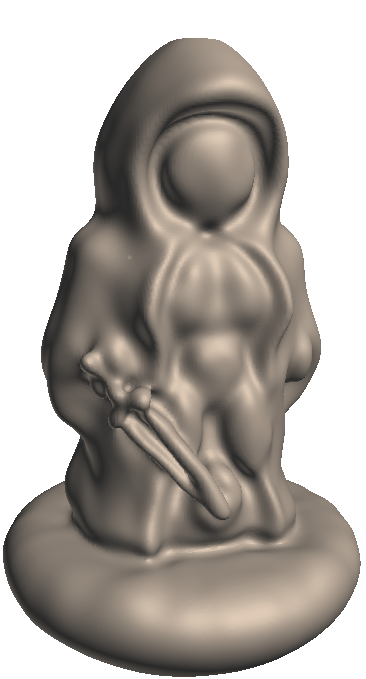}
    \includegraphics[width=0.30\columnwidth]{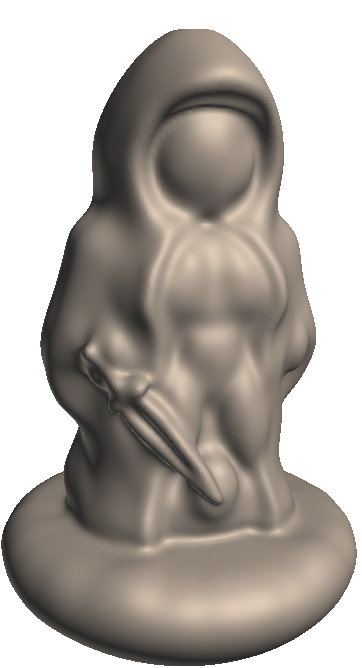}
    \includegraphics[width=0.29\columnwidth]{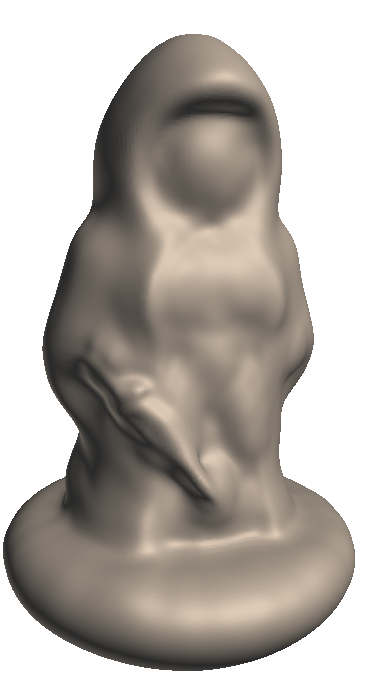}
    \includegraphics[width=0.335\columnwidth]{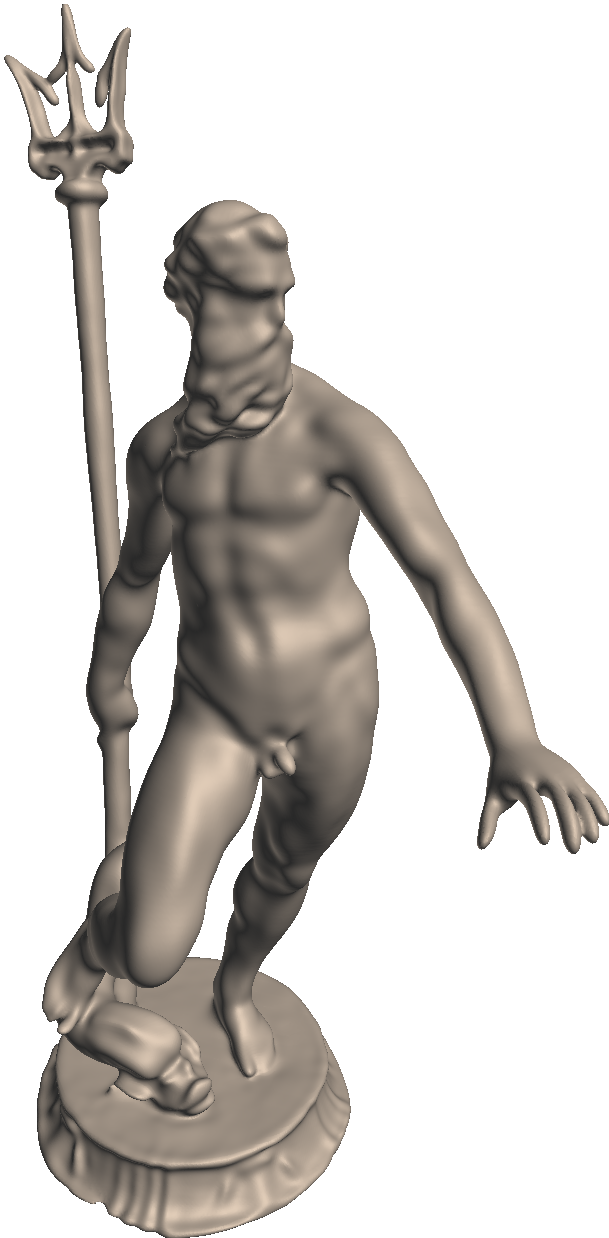}
    \includegraphics[width=0.32\columnwidth]{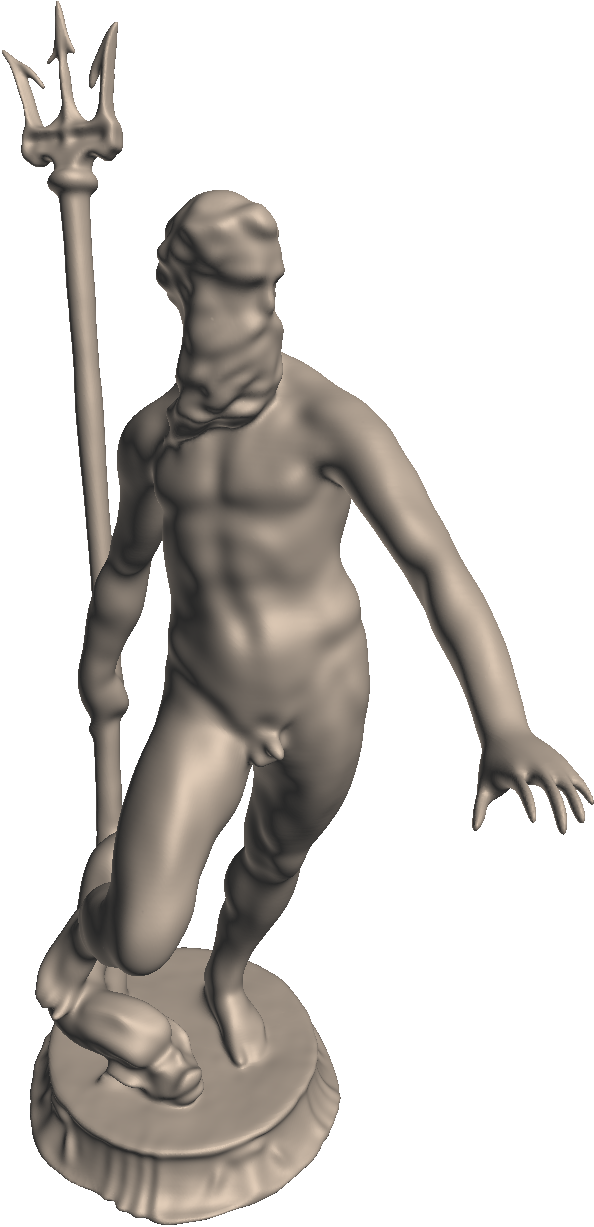}
    \includegraphics[width=0.31\columnwidth]{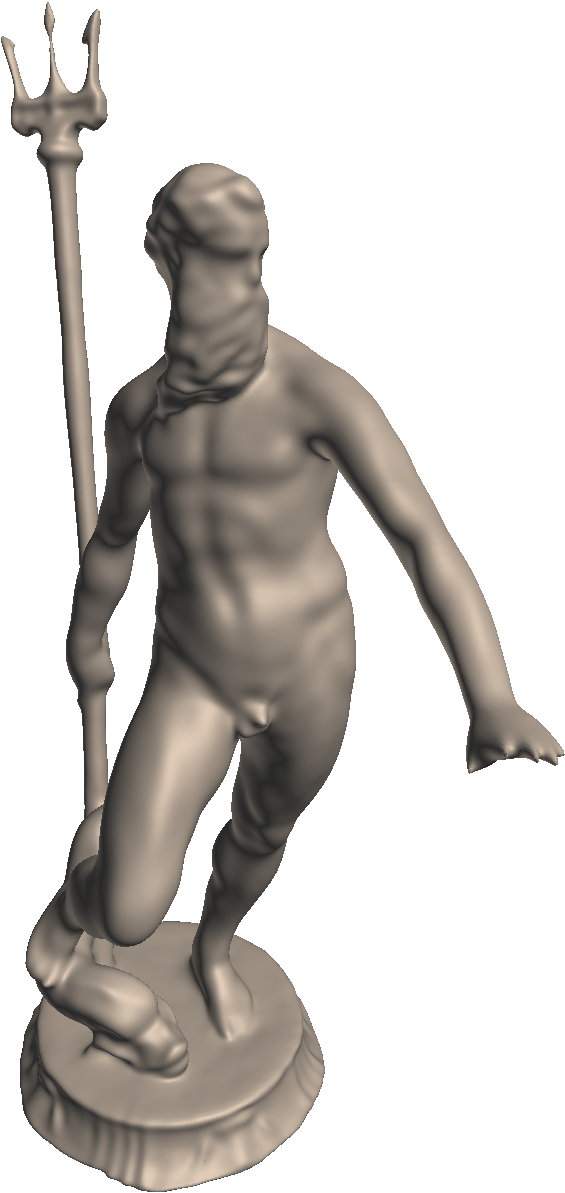}
    \caption{The zero-level sets of $f_\theta$ for Bob, Max Plank, Falcon, Witch, and Neptune models (middle). The left and right columns provide the sharpening and smoothing of the models.}
    \label{fig:init_conditions}
\end{figure}


\subsection{Additional interpolations}\label{s-interpolation}

Finally, Figure~\ref{s-interpolation} shows additional examples of interpolations between neural implicit functions using the method presented in Sec~5.3 of the paper.
The bracelets and chairs models are from the Thingi10K dataset~\cite{Thingi10K}.
We choose to interpolate the chairs because this was also an example considered by Lipschitz MLP~\cite{liu2022learning}.
We observed that when the features of the objects are aligned, the interpolation works like morphing between the shapes; see Line 3 of Figure~\ref{fig:interpolations}.
\begin{figure}[H]
    \centering
    \includegraphics[width=0.19\columnwidth]{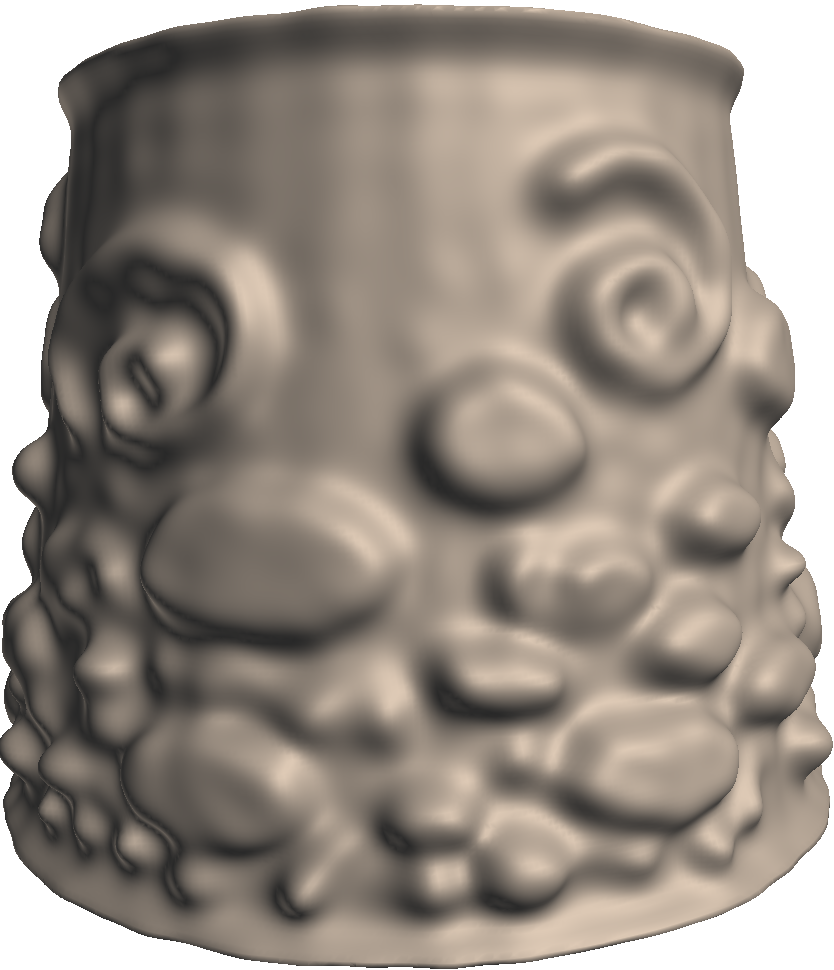}
    \includegraphics[width=0.19\columnwidth]{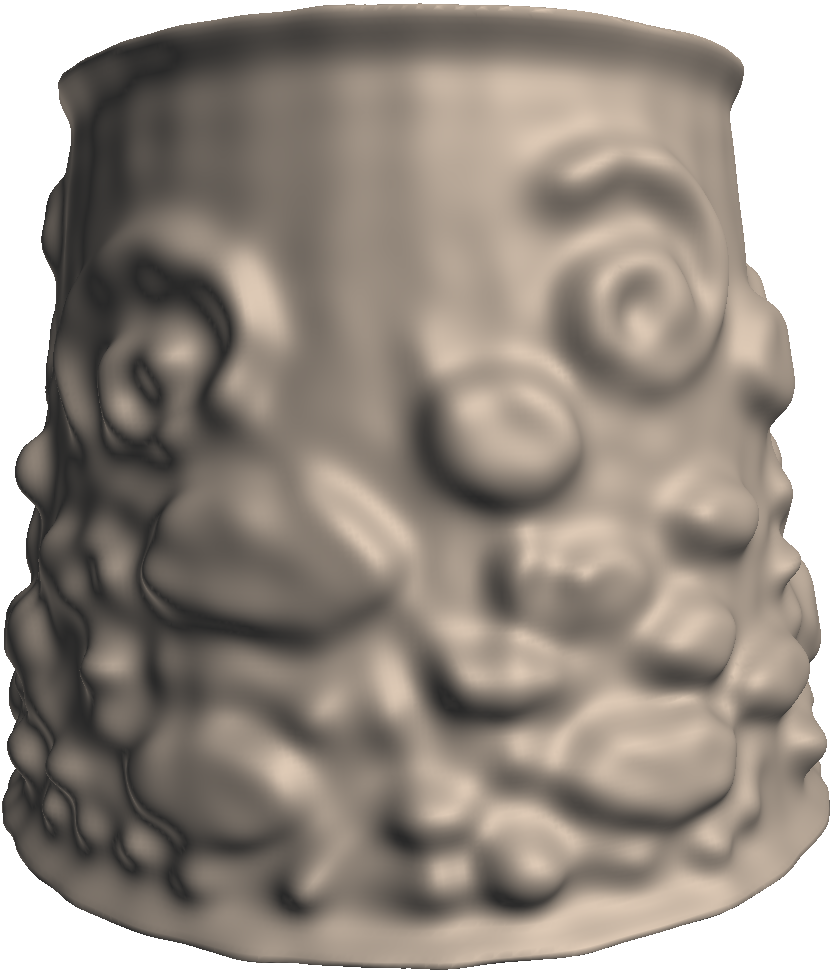}
    \includegraphics[width=0.19\columnwidth]{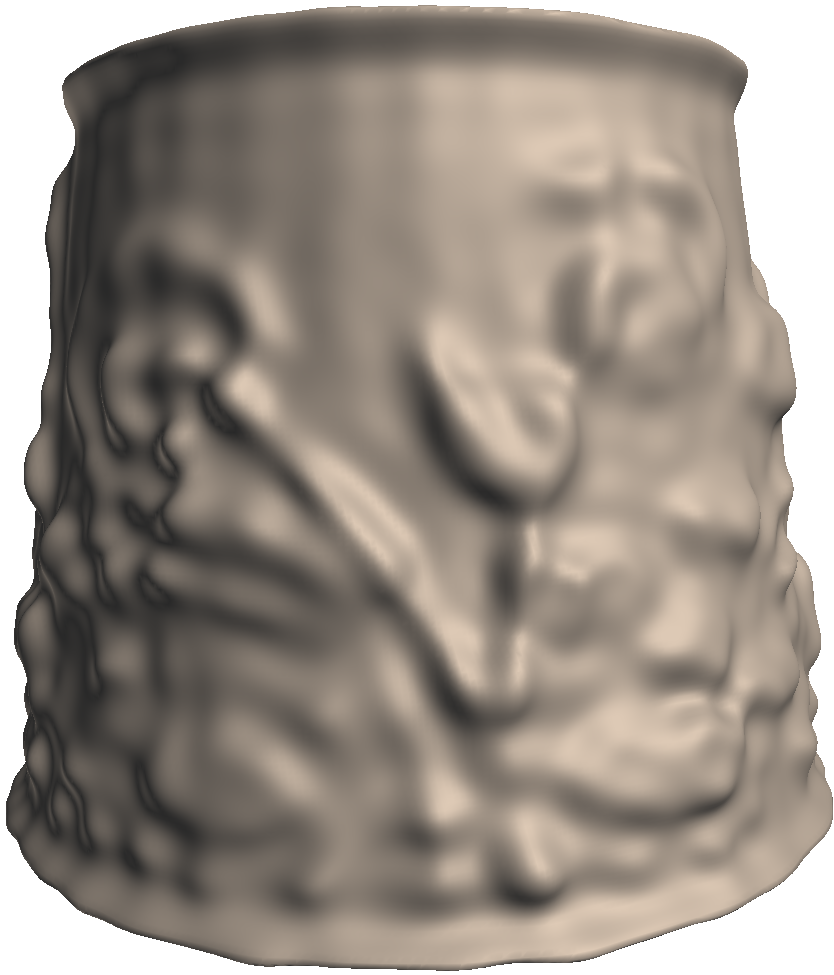}
    \includegraphics[width=0.19\columnwidth]{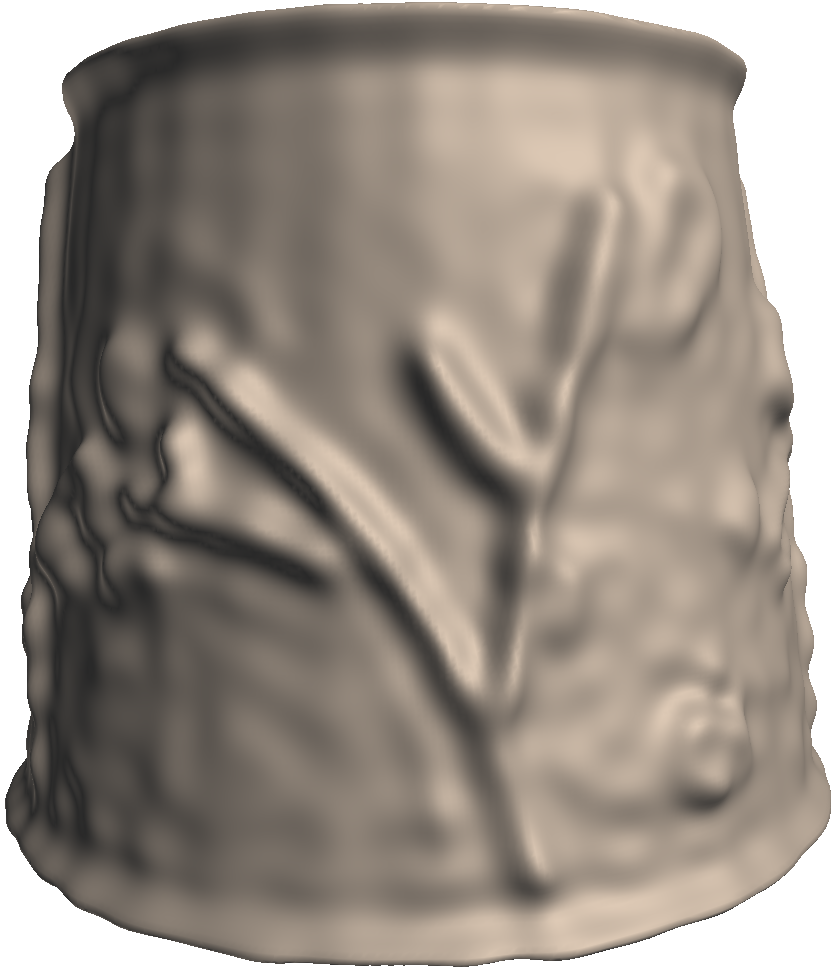}
    \includegraphics[width=0.19\columnwidth]{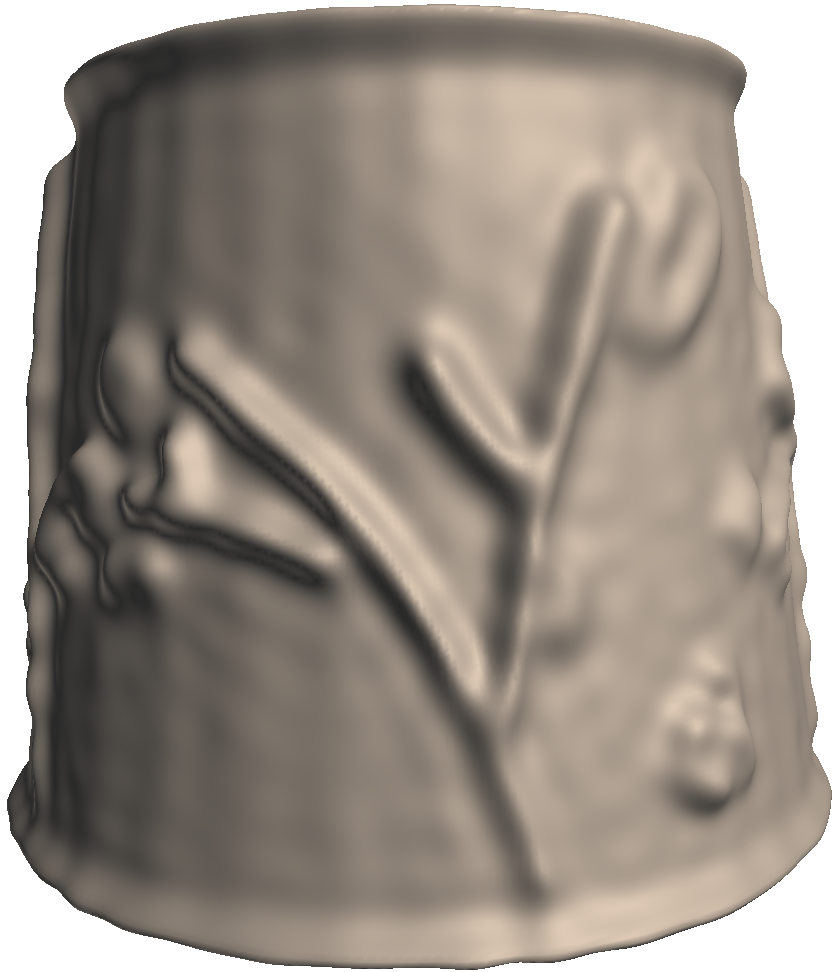}
    \includegraphics[width=0.19\columnwidth]{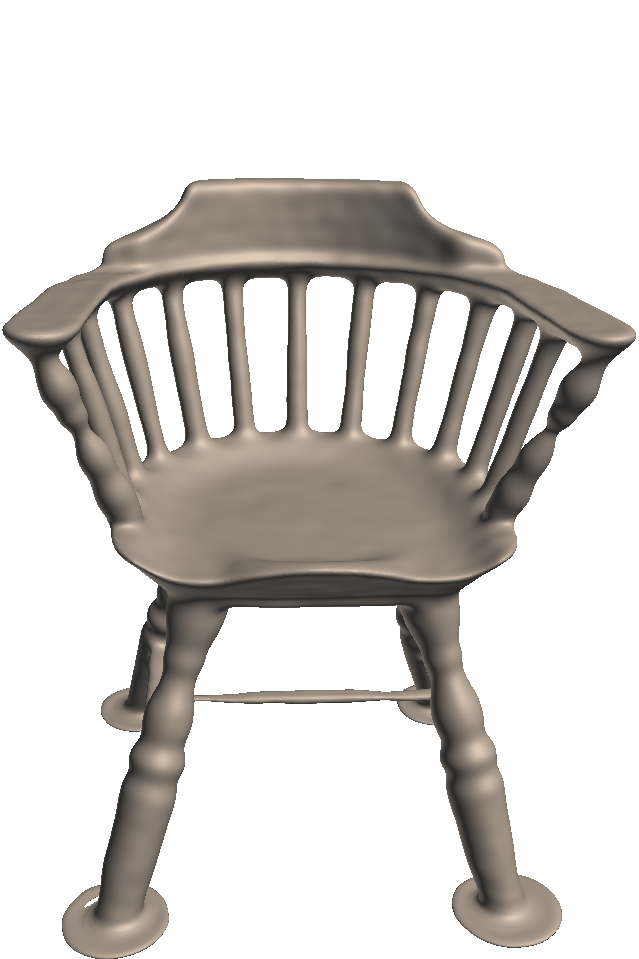}
    \includegraphics[width=0.19\columnwidth]{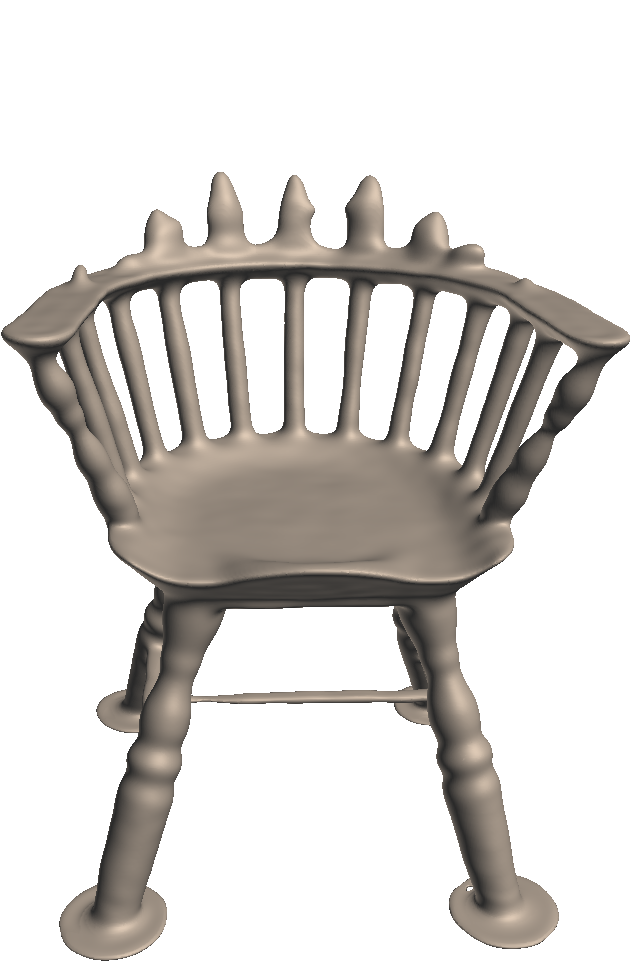}
    \includegraphics[width=0.19\columnwidth]{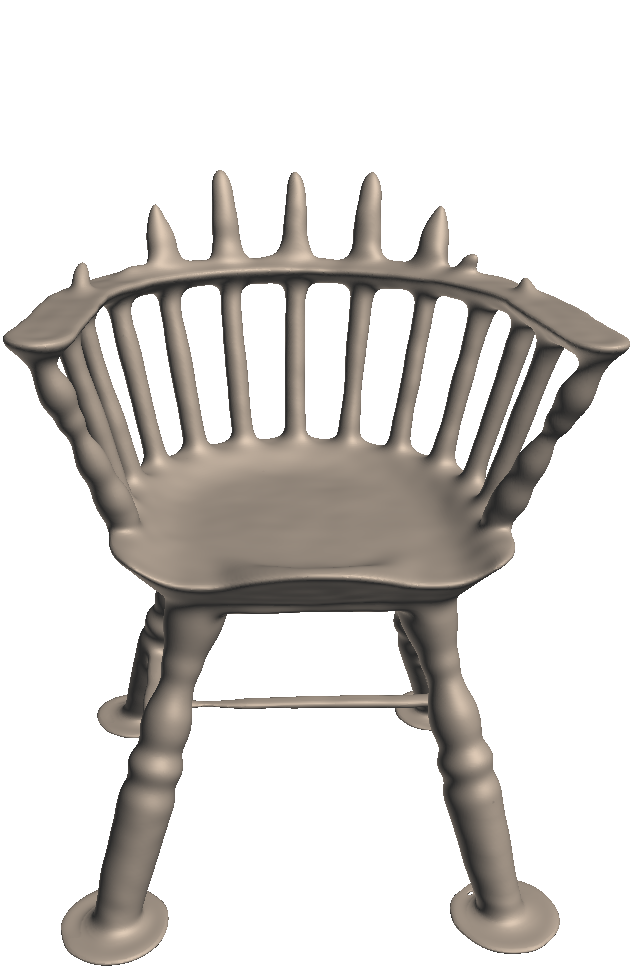}
    \includegraphics[width=0.19\columnwidth]{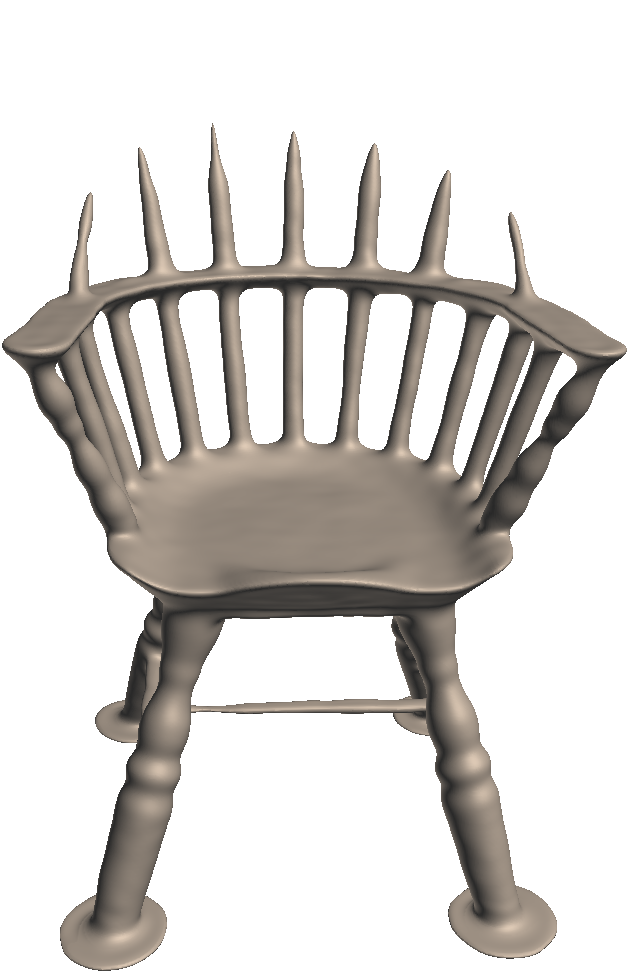}
    \includegraphics[width=0.19\columnwidth]{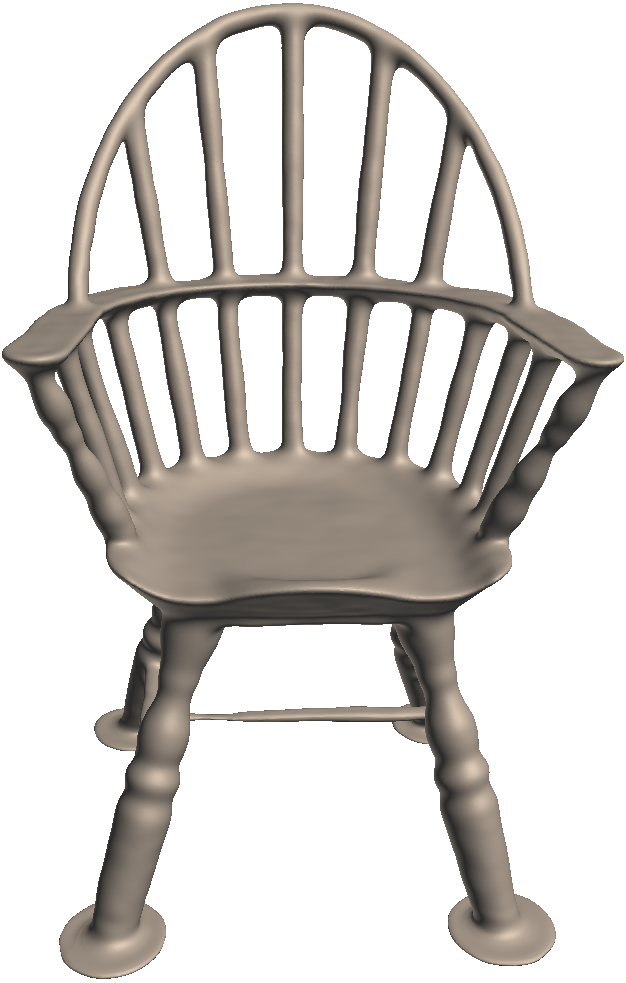}
    \includegraphics[width=0.19\columnwidth]{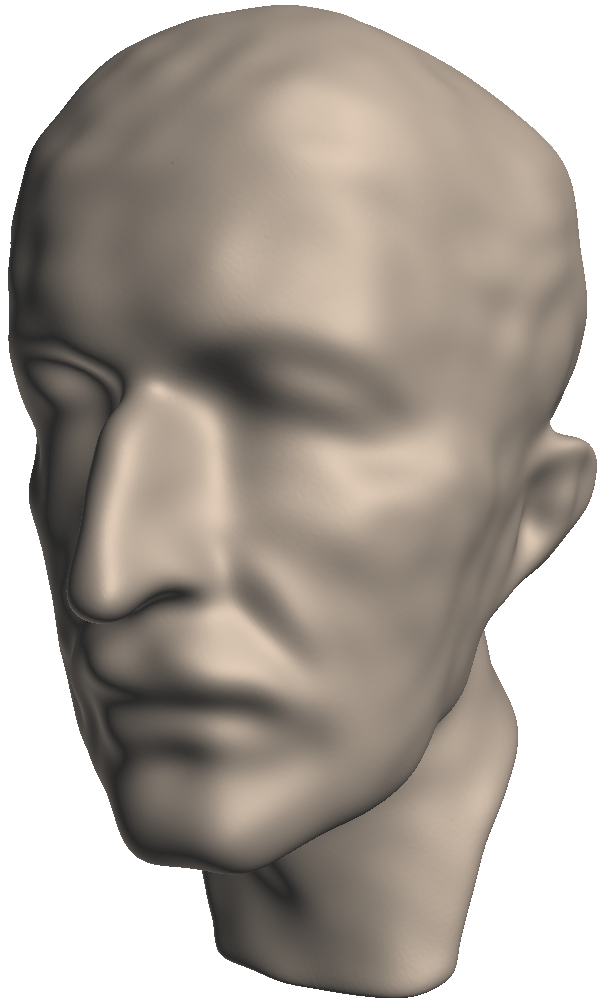}
    \includegraphics[width=0.19\columnwidth]{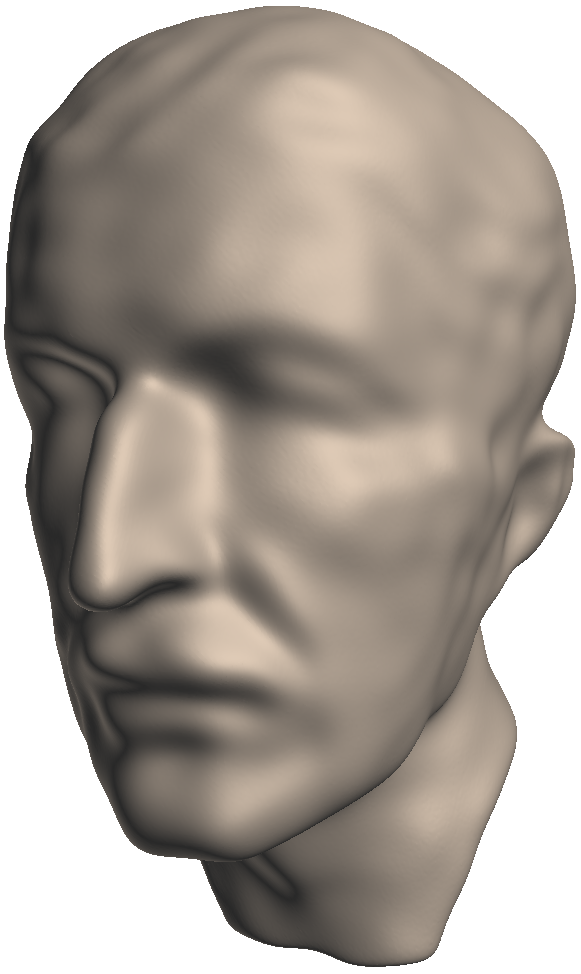}
    \includegraphics[width=0.19\columnwidth]{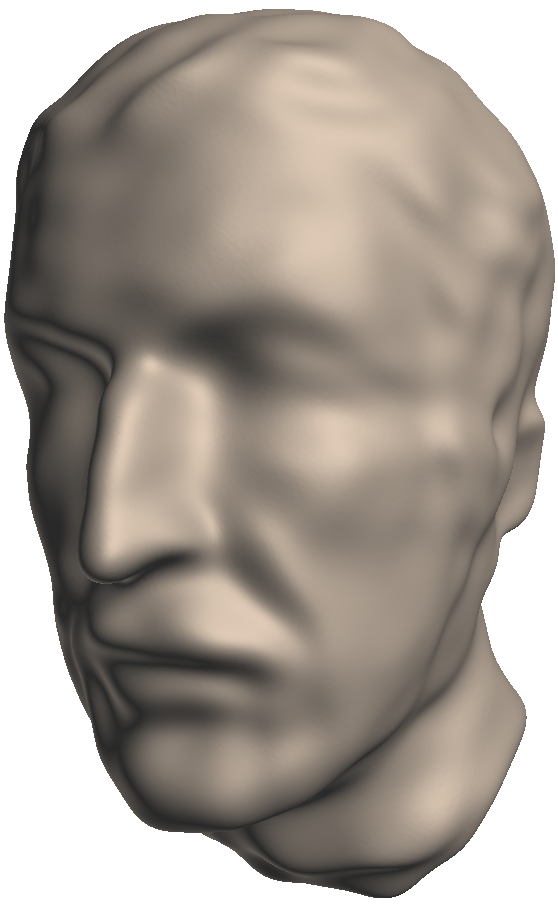}
    \includegraphics[width=0.19\columnwidth]{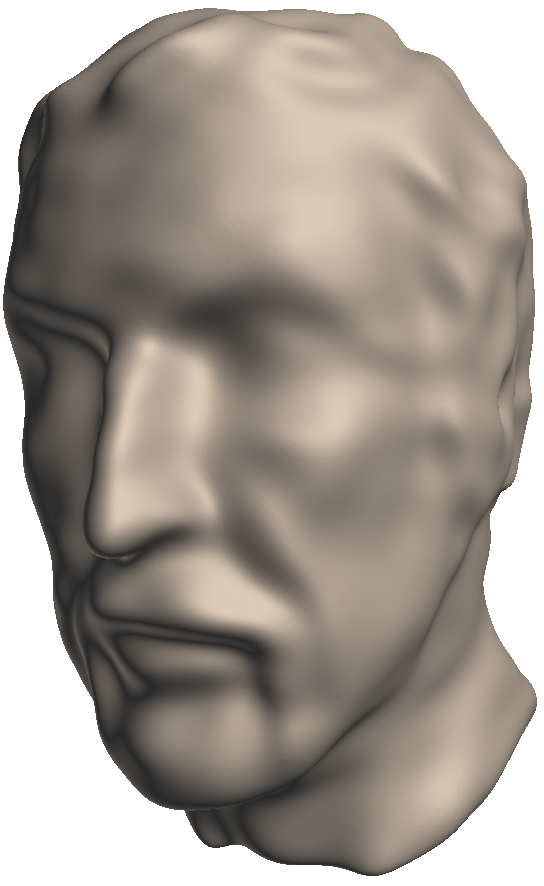}
    \includegraphics[width=0.19\columnwidth]{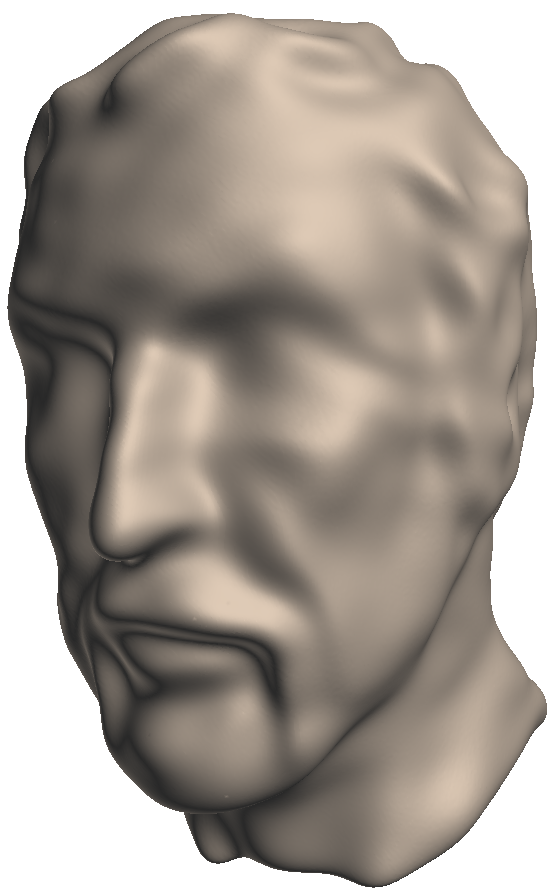}
    \caption{Interpolations between neural implicit functions. We observe that choosing surfaces with a significant overlapping of their interior regions results in better interpolations.}
    \label{fig:interpolations}
\end{figure}

{\small
\bibliographystyle{ieee_fullname}
\bibliography{egbib}
}